\documentclass[sigconf]{acmart}
\usepackage{booktabs} 
\usepackage{breqn}
\usepackage{graphicx}
\usepackage{xcolor}
\usepackage{enumitem}

\usepackage{tabularray}

\usepackage{array}
\usepackage[algo2e,ruled,vlined,lined,linesnumbered]{algorithm2e}
\newcommand{\nosemic}{\renewcommand{\@endalgocfline}{\relax}}
\newcommand{\dosemic}{\renewcommand{\@endalgocfline}{\algocf@endline}}

\usepackage{amsthm}
\usepackage{amsmath,bm}
\newtheorem{theorem}{Theorem}

\newtheorem{lemma}{Lemma}

\newtheorem{assumption}{Assumption}

\usepackage{amsmath}

\usepackage{enumitem,kantlipsum}
\usepackage{amsfonts}
\usepackage{nicefrac}       
\usepackage{float}
\usepackage{microtype}      

\usepackage{subcaption,epsfig,epstopdf,color}
\theoremstyle{definition}
\usepackage{makecell}

\copyrightyear{2023}
\acmYear{2023}
\setcopyright{rightsretained}
\acmConference[KDD '23]{Proceedings of the 29th ACM SIGKDD Conference on Knowledge Discovery and Data Mining}{August 6--10, 2023}{Long Beach, CA, USA}
\acmBooktitle{Proceedings of the 29th ACM SIGKDD Conference on Knowledge Discovery and Data Mining (KDD '23), August 6--10, 2023, Long Beach, CA, USA}
\acmDOI{10.1145/3580305.3599333}
\acmISBN{979-8-4007-0103-0/23/08}

\begin{document}

\title{Enhance Diffusion to Improve Robust Generalization}

\author{Jianhui Sun}
\affiliation{
\institution{
University of Virginia, js9gu@virginia.edu
}
\country{}
}

\author{Sanchit Sinha}
\affiliation{
\institution{
University of Virginia, ss7mu@virginia.edu
}
\country{}
}

\author{Aidong Zhang}
\affiliation{
\institution{
University of Virginia, aidong@virginia.edu
}
\country{}
}

\renewcommand{\shortauthors}{Jianhui Sun et al.}

\begin{abstract}
 Deep neural networks are susceptible to human imperceptible adversarial perturbations. One of the strongest defense mechanisms is \emph{Adversarial Training} (AT). In this paper, we aim to address two predominant problems in AT. First, there is still little consensus on how to set hyperparameters with a performance guarantee for AT research, and customized settings impede a fair comparison between different model designs in AT research. Second, the robustly trained neural networks struggle to generalize well and suffer from tremendous overfitting. This paper focuses on the primary AT framework - Projected Gradient Descent Adversarial Training (PGD-AT). We approximate the dynamic of PGD-AT by a continuous-time Stochastic Differential Equation (SDE), and show that the diffusion term of this SDE determines the robust generalization. An immediate implication of this theoretical finding is that robust generalization is positively correlated with the ratio between learning rate and batch size. We further propose a novel approach, \emph{Diffusion Enhanced Adversarial Training} (DEAT), to manipulate the diffusion term to improve robust generalization with virtually no extra computational burden. We theoretically show that DEAT obtains a tighter generalization bound than PGD-AT. Our empirical investigation is extensive and firmly attests that DEAT universally outperforms PGD-AT by a significant margin.
\end{abstract}



\maketitle

\section{Introduction}

\label{sec:intro}

Despite achieving surprisingly good performance in a wide range of tasks, deep learning models have been shown to be vulnerable to adversarial attacks which add human imperceptible perturbations that could significantly jeopardize the model performance \cite{Goodfellow2015ExplainingAH}. Adversarial training (AT), which trains deep neural networks on adversarially perturbed inputs instead of on clean data, is one of the strongest defense strategies against such adversarial attacks.

This paper mainly focuses on the primary AT framework - Projected Gradient Descent Adversarial Training (PGD-AT) \cite{madry2018towards}. Though many new improvements \footnote{e.g., training tricks including early stopping w.r.t. the training epoch \cite{Rice2020OverfittingIA}, and label smoothing \cite{Shafahi2019AdversarialTF} prove to be useful to improve robustness.} have been proposed on top of PGD-AT to be better tailored to the needs of different domains, or to alleviate the heavy computational burden \cite{tramer2018ensemble, wang2018a, Mao2019MetricLF, Carmon2019UnlabeledDI, hendrycks2019pretraining, Shafahi2019AdversarialTF, zhang2019you, Wong2020Fast}, PGD-AT at its vanilla version is still the default choice in most scenarios due to its compelling performances in several adversarial competitions \cite{Kurakin2018AdversarialAA, adversarial_vision_challenge}.

\subsection{Motivation}

\begin{figure}[t]

\centering
\includegraphics[width=7.5cm,height=6.3cm]{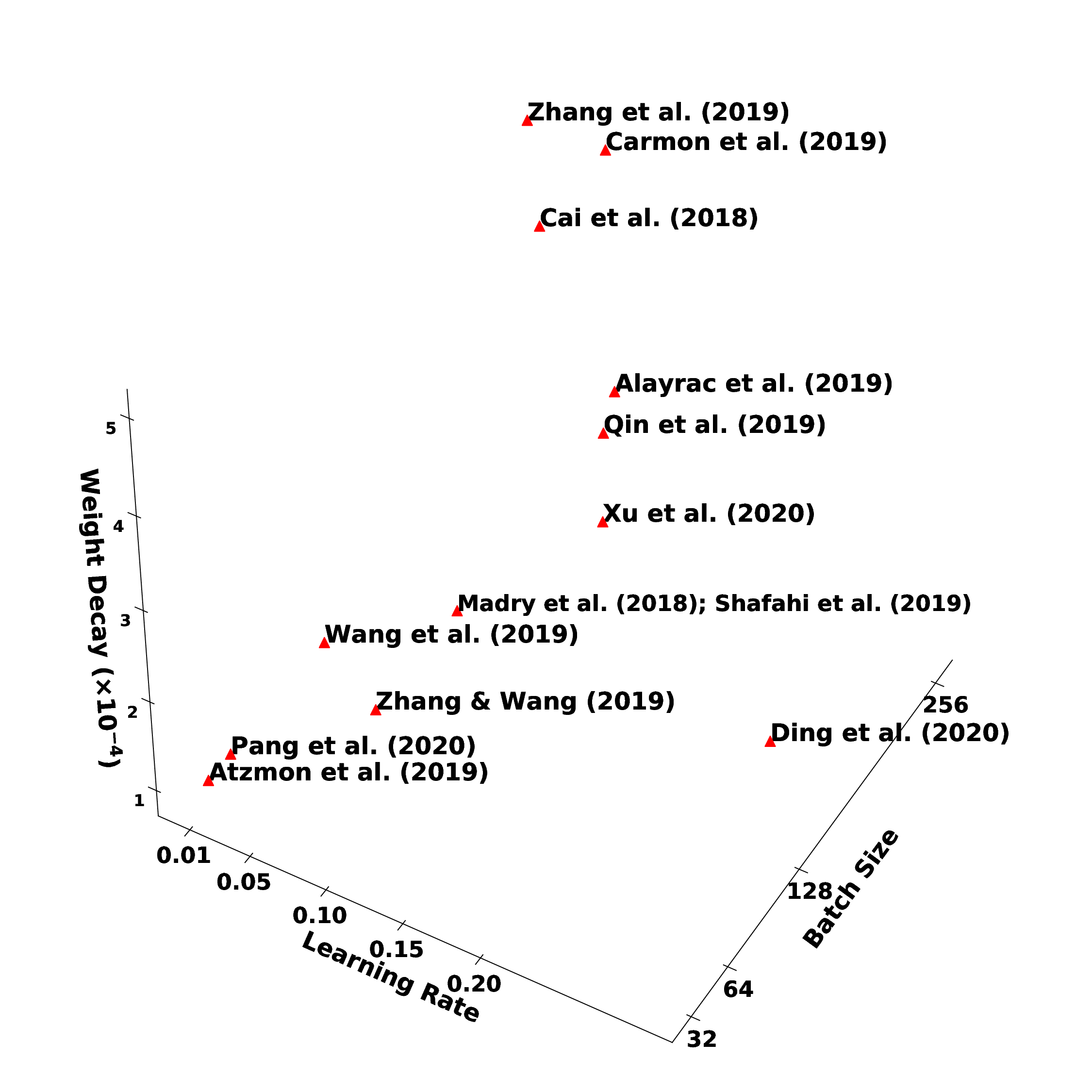}

\caption{\label{literature_fig} We summarize three key hyperparameters (learning rate, batch size, weight decay) used in a list of recent papers \cite{madry2018towards,Shafahi2019AdversarialTF,zhang2019you,Carmon2019UnlabeledDI,Cai2018Curriculum,Uesato2019AreLR,Qin2019AdversarialRT,Xu2020ExploringMR,pmlr-v97-wang19i,Zhang2019DefenseAA,Ding2020MMA,Pang2020Rethinking,atzmon2019controlling}. The hyperparameter specifications are highly inconsistent and a fair comparison is difficult in such condition as we will demonstrate in our empirical experiments, these hyperparameters make a huge difference in robust generalization.}

\end{figure}

This paper aims to address the following problems in AT:

\begin{enumerate}[leftmargin=*]
    \item[I.] \textbf{Inconsistent hyperparameter specifications impede a fair comparison between different model designs.} 
    
    Though the configuration of hyperparameters is known to play an essential role in the performance of AT, there is little consensus on how to set hyperparameters with a performance guarantee. For example in Figure \ref{literature_fig}, we plot a list of recent AT papers on the (learning rate, batch size, weight decay) space according to each paper's specification and we could observe that the hyperparameters of each paper are relatively different from each other with little consensus. Moreover, the completely customized settings make it extremely difficult to understand which approach really works, as the misspecification of hyperparameters would potentially cancel out the improvements from the methods themselves. Most importantly, the lack of theoretical understanding also exhausts practitioners with time-consuming tuning efforts.

    \item[II.] \textbf{The robust generalization gap in AT is surprisingly large.} 
    
    Overfitting is a dominant problem in adversarially trained deep networks \cite{Rice2020OverfittingIA}. To demonstrate that, we run both standard training (non-adversarial) and adversarial training on CIFAR10 with VGG \cite{Simonyan14VGG} and SENet \cite{CVPR18SENet}. Training curves are reported in Figure \ref{large_gap_fig}.
    We could observe the robust test accuracy is much lower than the standard test accuracy. Further training will continue to improve the robust training loss of the classifier, to the extent that robust training loss could closely track standard training loss \cite{Schmidt18MoreData}, but fail to further improve robust testing loss. Early stopping is advocated to partially alleviate overfitting \cite{Zhang2019TheoreticallyPT,Rice2020OverfittingIA}, but there is still huge room for improvement.

\end{enumerate}

\subsection{Contribution}

In this paper, to address the aforementioned problems, we consider PGD-AT as an alternating stochastic gradient descent. Motivated by the theoretical attempts which approximate the discrete-time dynamic of stochastic gradient algorithm with continuous-time Stochastic Differential Equation (SDE) \cite{Mandt17SGDApproximateBayesian,Li17SME,zhu2019anisotropic,He19ControlBatch}, we derive the continuous-time SDE dynamic for PGD-AT. The SDE contains a drift term and a diffusion term, and we further prove the diffusion term determines the robust generalization performance.

\begin{figure}[h]
\centering

\includegraphics[width=6.0cm]{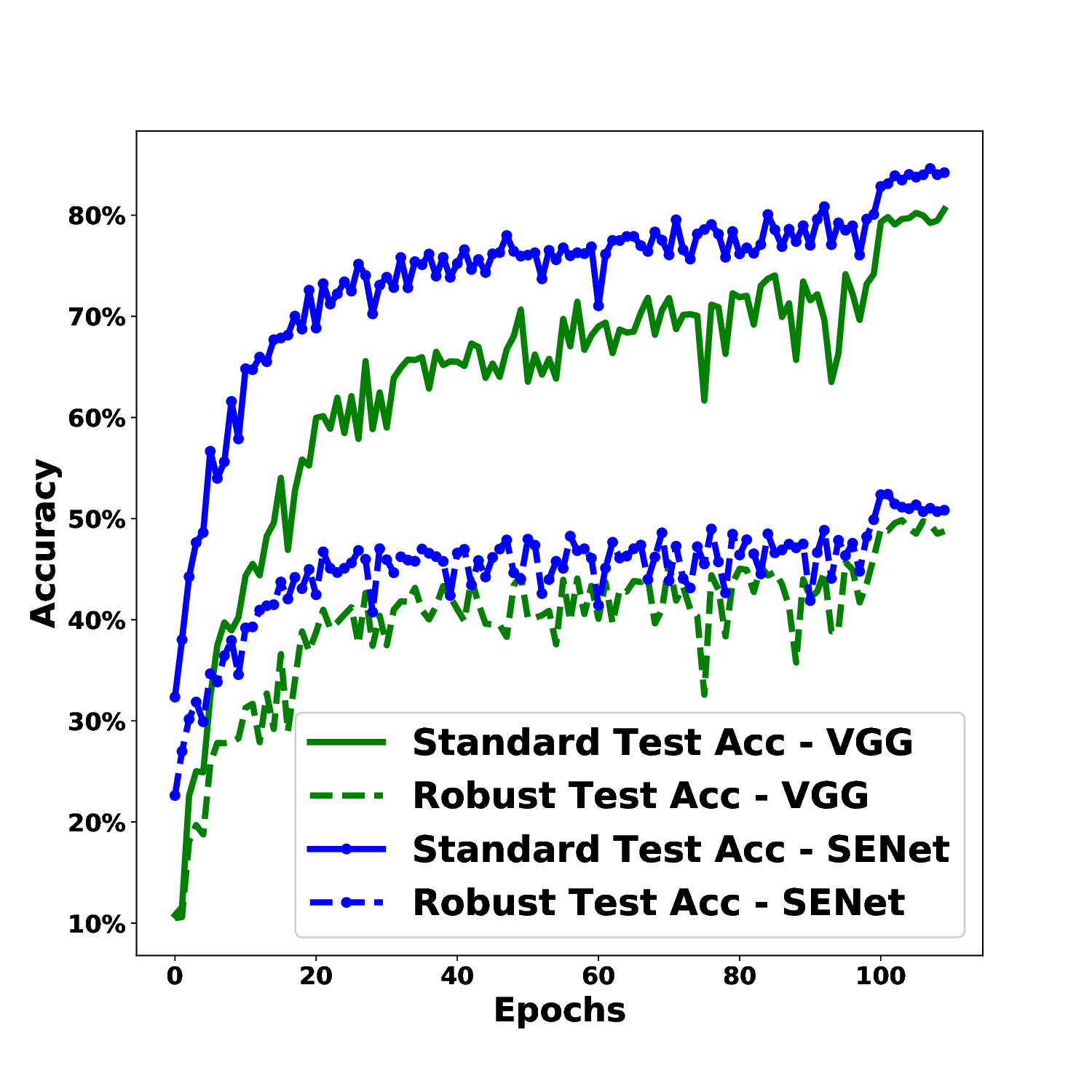}

\caption{\label{large_gap_fig} Classification accuracy for standard training (non-adversarial) and adversarial training on CIFAR10 with VGG and SENet. Table \ref{default_setting_table} summarizes the experimental setting. The adversarial test accuracy (dashed line) is far from the non-adversarial test accuracy (solid line) in both architectures. The generalization gap for the robust accuracy is significant and much larger than normal training. 
}

\end{figure}

As the diffusion term is determined by (A) ratio of learning rate $\alpha$ and batch size $b$ and (B) gradient noise, an immediate implication of our theorem is that the robust generalization has a positive correlation with the size of both (A) and (B). In other words,  we could improve robust generalization via scaling up (A) and (B).
Although it is fairly simple to scale up (A) by increasing $\alpha$ and decreasing $b$,  adjusting $\alpha$ and $b$ could be a double-edged sword. One reason is that small batch improves generalization while significantly increases training time. Considering the computational cost of adversarial training is already extremely expensive (e.g., the PGD-10 training of ResNet on CIFAR-10 takes several days on a single GPU), large batch training is apparently more desirable. $\alpha$ is allowed to increase only within a very small range to ensure convergence of AT algorithm.

To overcome the aforementioned limitations, we propose a novel algorithm, DEAT (\emph{Diffusion Enhanced Adversarial Training}), to instead adjust (B) to improve robust generalization (see Algorithm \ref{Diffusion_PGD_AT_algorithm}). Our approach adds virtually no extra computational burden, and universally achieves better robust testing accuracy over vanilla PGD-AT by a large margin. We theoretically prove DEAT achieves a tighter robust generalization gap. Our extensive experimental investigation strongly supports our theoretical findings and attests the effectiveness of DEAT. 

We summarize our contributions as follows:

\begin{enumerate}[leftmargin=*]
    \item[I.] Theoretically, we approximate PGD-AT with a continuous-time SDE, and prove the diffusion term of this SDE determines the robust generalization. The theorem guides us how to tune $\alpha$ and $b$ in PGD-AT. To our best knowledge, this is the first study that rigorously proves the role of hyperparameters in AT.
    
    \item[II.] Algorithmically, we propose a novel approach, DEAT (Diffusion Enhanced Adversarial Training), to manipulate the diffusion term with virtually no additional computational cost, and manage to universally improve over vanilla PGD-AT by a significant margin. We also theoretically show DEAT is guaranteed to generalize better than PGD-AT. Interestingly, DEAT also improves the generalization performance in non-adversarial tasks, which further verifies our theoretical findings.
    
\end{enumerate}

\textbf{Organization} In Section \ref{backgroun_section}, we formally introduce adversarial training and PGD-AT, which are pertinent to this work. In Section \ref{generalization_section}, we present our main theorem that derives the robust generalization bound of PGD-AT. In Section \ref{algorithm_section}, motivated by the theoretical findings and in recognition of the drawbacks in adjusting $\alpha$ and $b$, we present our novel DEAT (Diffusion Enhanced Adversarial Training). We theoretically show DEAT has a tighter generalization bound. In Section \ref{experimental_evidence}, we conduct extensive experiments to verify our theoretical findings and the effectiveness of DEAT. Related works are discussed in Section \ref{related_work}. Proofs of all our theorems and corollaries are presented in Appendix.

\section{Background: PGD-AT}
In this section, we formally introduce PGD-AT which is the main focus of this work.

\label{backgroun_section}

\textbf{Notation:} This paragraph summarizes the notation used throughout the paper. Let $\theta$, $\mathcal{D}$, and $l_\theta(x_i,y_i)$ be the trainable model parameter, data distribution, and loss function, respectively. Let $\{z_i=(x_i,y_i)\}_{i=1}^N$ denote the training set, and $\{x_i\}_{i=1}^N\subset\mathbb{R}^d$. Expected risk function is defined as $\mathcal{R}(\theta)\triangleq\mathbb{E}_{z\sim\mathcal{D}} l_\theta(z)$. Empirical risk $\mathcal{R}_\zeta(\theta)$ is an unbiased estimator of the expected risk function, and is defined as $\mathcal{R}_\zeta(\theta)\triangleq \frac{1}{b}\sum_{j\in\zeta} \mathcal{R}_j(\theta)$, where $\mathcal{R}_j(\theta)\triangleq l_\theta(z_j)$ is the contribution to risk from $j$-th data point. $\zeta$ represents a mini-batch of random samples and $b\triangleq\lvert\zeta\rvert$ represents the batch size. Similarly, we define $\nabla_{\theta}\mathcal{R}$, $\nabla_{\theta}\mathcal{R}_j$, and $\nabla_{\theta}\mathcal{R}_{\zeta}$ as their gradients, respectively. We denote the empirical gradient as $\hat{g}(\theta)\triangleq\nabla_{\theta}\mathcal{R}_{\zeta}$ and exact gradient as $g(\theta)\triangleq\nabla_{\theta}\mathcal{R}$ for the simplicity of notation. 

In standard training, most learning tasks could be formulated as the following optimization problem:
\begin{equation}
\label{normal_training_task}
    \begin{gathered}
    \min_{\theta}\mathcal{R}(\theta)=\min_{\theta}\mathbb{E}_{(x_i,y_i)\sim\mathcal{D}} l_\theta(x_i,y_i),
    \end{gathered}
\end{equation}

Stochastic Gradient Descent (SGD) and its variants are most widely used to optimize \eqref{normal_training_task}. SGD updates with the following rule:
\begin{equation}
\label{sgd_normal}
    \begin{gathered}
    \theta_{t+1}=\theta_t-\alpha_t s_t,
    \end{gathered}
\end{equation}
where $\alpha_t$ and $s_t$ are the learning rate and search direction at $t$-th step, respectively. SGD uses $\hat{g}_t\triangleq\hat{g}(\theta_t)$ as $s_t$. 

The performance of learning models, depends heavily on whether SGD is able to reliably find a solution of \eqref{normal_training_task} that could generalize well to unseen test instances.


An adversarial attacker aims to add a human imperceptible perturbation to each sample, i.e., transform $\{z_i=(x_i,y_i)\}_{i=1}^N$ to $\{\Tilde{z}_i=(\Tilde{x}_i=x_i+\delta_i,y_i)\}_{i=1}^N$, where perturbations $\{\delta_i\}_{i=1}^N$ are constrained by a pre-specified budget $\Delta$ ($\delta_i\in\Delta$), such that the loss $l_\theta(\Tilde{x}_i,y_i)$ is large. The choice of budget is flexible. A typical formulation is $\{\delta\in\mathbb{R}^d:\lVert\delta\rVert_p\leq\epsilon\}$ for $p=1,2,\infty$. In order to defend such attack, we resort to solving the following objective function:

\begin{equation}
\label{adversarial_training_task}
    \begin{gathered}
    \min_{\theta\in\mathbb{R}_\theta}\rho(\theta),~ \text{where}~\rho(\theta)=\mathbb{E}_{(x_i,y_i)\sim\mathcal{D}}[\max_{\delta_i\in\Delta}l_\theta(x_i+\delta_i,y_i)] 
    \end{gathered}
\end{equation}

Objective function \eqref{adversarial_training_task} is a composition of an inner maximization problem and an outer minimization problem. The inner maximization problem simulates an attacker who aims to find an adversarial version of a given data point $x_i$ that achieves a high loss, while the outer minimization problem is to find model parameters so that the “adversarial loss” given by the inner attacker is minimized. Projected Gradient Descent Adversarial Training (PGD-AT) \cite{madry2018towards} solves this min-max game by gradient ascent on the perturbation parameter $\delta$ before applying gradient descent on the model parameter $\theta$.

The detailed pseudocode of PGD-AT
is in Algorithm \ref{PGD_AT_algorithm}. Basically, projected gradient descent (PGD) is applied $K$ steps on the negative loss function to produce strong adversarial examples in the inner loop, which can be viewed as a multi-step variant of Fast Gradient Sign Method (FGSM) \cite{Goodfellow2015ExplainingAH}, while every training example is replaced with its PGD-perturbed counterpart in the outer loop to produce a model that an adversary could not find adversarial examples easily.

\begin{algorithm2e}
\SetAlgoVlined
\KwIn{Loss function $l_\theta(z_i)$, initialization $\theta_0$, total training steps $T$, PGD steps $K$, inner/outer learning rates $\alpha_{I}$/$\alpha_{O}$, batch size $b$, perturbation budget set $\Delta$;}
\SetAlgoLined
\For{$t\in\{1,2,...,T\}$}
{
    Sample a mini-batch of random examples $\zeta=\{(x_{i_j},y_{i_j})\}_{j=1}^{b}$\;
    Set $\delta_0=0,\hat{x}_j=x_{i_j}$\;
    \For{$k\in\{1,...,K\}$}
    {
    $\delta_k=\Pi_{\Delta}(\delta_{k-1}+\frac{\alpha_{I}}{b}\sum_{j=1}^b\nabla_x l_{\theta_{t-1}}(\hat{x}_j+\delta_{k-1},y_{i_j}))$\;
    }
    $\theta_t=\theta_{t-1}-\frac{\alpha_{O}}{b}\sum_{j=1}^b\nabla_\theta l_{\theta_{t-1}}(\hat{x}_j+\delta_{K},y_{i_j}))$\;
}
return $\theta_T$
\caption{PGD-AT (Projected Gradient Descent Adversarial Training) \cite{madry2018towards}}
\label{PGD_AT_algorithm}
\end{algorithm2e}

\section{Theory: Robust Generalization Bound of PGD-AT}
\label{generalization_section}

In this section, we describe our logical framework of deriving the robust generalization gap of PGD-AT, and then identify the main factors that determine the generalization. 

To summarize the entire section before we dive into details, we consider PGD-AT as an alternating stochastic gradient descent and approximate the discrete-time dynamic of PGD-AT with continuous-time Stochastic Differential Equation (SDE), which contains a drift term and a diffusion term, and we would show the diffusion term determines the robust generalization. Our theorem immediately points out the robust generalization has a positive correlation with the ratio between learning rate $\alpha$ and batch size $b$.

Let us first introduce our logical framework in Section \ref{roadmap_subsec} before we present main theorem in Section \ref{main_theorem_sec}.

\subsection{Roadmap to robust generalization bound}
\label{roadmap_subsec}

\vspace*{6pt}
\noindent\textbf{Continuous-time dynamics of gradient based methods}
\vspace*{6pt}

A powerful analysis tool for stochastic gradient based methods is to model the continuous-time dynamics with stochastic differential equations and then study its limit behavior \cite{Mandt17SGDApproximateBayesian,Li17SME,hu2018diffusion,He19ControlBatch,Sun21KDDHyperparameter,xie2021positive}.
\cite{Mandt17SGDApproximateBayesian} characterizes the continuous-time dynamics of using a constant step size SGD \eqref{sgd_normal} to optimize normal training task \eqref{normal_training_task}.

\begin{lemma}[\cite{Mandt17SGDApproximateBayesian}]
\label{sgd_sde_lemma}
Assume the risk function \footnote{Without loss of generality, we assume the minimum of the risk function is at $\theta = 0$, as we could always translate the minimum to 0.} is locally quadratic, and gradient noise is Gaussian with mean 0 and covariance $\frac{1}{b}H$, and $H=BB^T$ for some $B$. The following two statements hold,
\begin{enumerate}[leftmargin=*]

\item[I.] Constant-step size SGD \eqref{sgd_normal} could be recast as a discretization of the following continuous-time dynamics:

\begin{equation}
\label{sgd_sde_formation}
    \begin{gathered}
     d\theta=-\alpha g(\theta)dt+\frac{\alpha}{\sqrt{b}}BdW_t
     \end{gathered}
\end{equation}
where $dW_t=\mathcal{N}(0,Idt)$ is a Wiener process.

\item[II.] The stationary distribution of stochastic process \eqref{sgd_sde_formation} is Gaussian and its covariance matrix $Q$ is explicit.
\end{enumerate}

\end{lemma}

$\alpha g(\theta)$ and $\frac{\alpha}{\sqrt{b}}B$ are referred to as drift and diffusion, respectively. Many variants of SGD (e.g. heavy ball momentum \cite{POLYAK1964HeavyBall} and Nesterov's accelerated gradient \cite{Nesterov1983NAG}) can also be cast as a modified version of \eqref{sgd_sde_formation}, and we could explicitly write out its stationary distribution as well \cite{Gitman19Momentum}. 

Next we will discuss PAC-Bayesian generalization bound and how it connects to the SDE approximation.

\vspace*{6pt}
\noindent\textbf{PAC-Bayesian generalization bound}
\vspace*{6pt}

Bayesian learning paradigm studies a distribution of every possible setting of model parameter $\theta$ instead of betting on one single setting of parameters to manage model uncertainty, and has proven increasingly powerful in many applications. In Bayesian framework, $\theta$ is assumed to follow some prior distribution $P$ (reflects the prior knowledge of model parameters), and at each iteration of SGD \eqref{sgd_normal} (or any other stochastic gradient based algorithm), the $\theta$ distribution shifts to $\{Q_t\}_{t\ge0}$, and converges to posterior distribution $Q$ (reflects knowledge of model parameters after learning with $\mathcal{D}$). 

Bayesian risk function is defined as $\mathcal{R}(Q)\triangleq\mathbb{E}_{\theta\sim Q}\mathbb{E}_{(x,y)\sim\mathcal{D}} l(f_\theta(x),y)$, and $\hat{\mathcal{R}}(Q)\triangleq\mathbb{E}_{\theta\sim Q}\frac{1}{N}\sum_{j=1}^N l(f_\theta(x_j),y_j)$. $\mathcal{R}(Q)$ is the population risk, while $\hat{\mathcal{R}}(Q)$ is the risk evaluated on the training set and $N\triangleq\lvert\mathcal{D}\rvert$ is the sample size. The generalization bound could therefore be defined as follows:
\begin{gather}
\label{simple_gen_bound_def}
    \mathcal{E}\triangleq\lvert\mathcal{R}(Q)-\hat{\mathcal{R}}(Q)\rvert.
\end{gather}

The following lemma connects generalization bound to the stationary distribution of a stochastic gradient algorithm.

\begin{lemma}[PAC-Bayesian Generalization Bound \cite{Williamson97PAC-Bayes,McAllester98PAC-Bayes}]
\label{PAC_Bayes_theorem}
Let $\text{KL}(Q||P)$ be the Kullback-Leibler divergence between distributions $Q$ and $P$. For any positive real $\varepsilon\in(0,1)$, with probability at least $1-\varepsilon$ over a sample of size $N$, we have the following inequality for all distributions $Q$:
\begin{gather}
    \lvert\mathcal{R}(Q)-\hat{\mathcal{R}}(Q)\rvert\le\sqrt{\frac{\text{KL}(Q||P)+\log\frac{1}{\varepsilon}+\log N+2}{2N-1}}
\end{gather}
\end{lemma}

Let $\mathcal{G}$ denote $\sqrt{\frac{\text{KL}(Q||P)+\log\frac{1}{\varepsilon}+\log N+2}{2N-1}}$, which is an upper bound of generalization error, i.e. $\mathcal{G}$ is an upper bound of $\mathcal{E}$ in \eqref{simple_gen_bound_def}. The prior $P$ is typically assumed to be a Gaussian prior $\mathcal{N}(\theta_0,\lambda_0 I_d)$, reflecting the common practice of Gaussian initialization \cite{Du2019Optim} and $L_2$ regularization, and posterior distribution $Q$ is the stationary distribution of the stochastic gradient algorithm under study.

The importance of Lemma \ref{PAC_Bayes_theorem} is that, we could easily get an upper bound of generalization bound if we could explicitly represent $\text{KL}(Q||P)$. Recall Lemma \ref{sgd_sde_lemma} gives the exact form of $Q$ for SGD, and therefore, naturally results in a generalization bound.

\subsection{Robust generalization of PGD-AT}
\label{main_theorem_sec}

SGD \eqref{sgd_normal} can be viewed as a special example (by setting the total steps of PGD attack $K=0$) of PGD-AT (Algorithm \ref{PGD_AT_algorithm_revised}) \footnote{As a matter of fact, all our theoretical findings in this paper also apply for non-AT settings, i.e., ordinary learning tasks without adversarial attacks. Interestingly, our proposed approach DEAT is not only capable of improving robust generalization in AT, but also improving generalization performance in ordinary learning tasks by setting $K=0$ compared to SGD. Please refer to Section \ref{subsec:deat_non_at} for more details.}. PGD-AT is also a stochastic gradient based iterative updating process. Therefore, a natural question arises:

\begin{center}
    \textit{Can we approximate the continuous-time dynamics of PGD-AT by a stochastic differential equation?}
\end{center}

We provide a positive answer to this question in Theorem \ref{pgd_at_generalization_theorem}. However, general SDEs do not possess closed-form stationary distributions, which makes the downstream tasks extremely difficult to proceed. The following question requires answering:

\begin{center}
    \textit{Can we explicitly represent the stationary distribution of this SDE and subsequently calculate $\text{KL}(Q||P)$ required in Lemma \ref{PAC_Bayes_theorem}?}
\end{center}

The answer also has a positive answer with mild assumptions. With the stationary distribution, we will leverage Lemma \ref{PAC_Bayes_theorem} to derive a generalization bound of PGD-AT, which would be a powerful analytic tool to identify the main factors that determine the robust generalization.

We are now ready to give our main theorem.

\begin{theorem}
\label{pgd_at_generalization_theorem}
Assume the risk function is locally quadratic, and gradient noise is Gaussian \footnote{Please refer to Section \ref{subsec:analysis_assump} for more details}. Suppose inner learning rate equals outer learning rate, and they are both fixed, i.e., $\alpha=\alpha_I=\alpha_O$. Let the Hessian matrix of risk function be $A$, and covariance matrix of Gaussian noise be $H=BB^T$. Let $b$ denote the batch size and $\mathcal{G}$ be the upper bound of generalization error.

The following statements hold,

1. The continuous-time dynamics of PGD-AT can be described by the following stochastic process:

\begin{equation}
\label{adversarial_training_sde}
    \begin{gathered}
    d\theta = fdt + \sigma dW_t, \quad \text{where} \quad dW_t=\mathcal{N}(0,Idt) \quad \text{is Wiener process},\\
    f = -(A+(K+\frac{1}{2})\alpha A^2)\theta \qquad \text{and} \quad \sigma = \sqrt{\frac{\alpha}{b}}AB
    \end{gathered}
\end{equation}
$f$ and $\sigma$ are referred to as drift term and diffusion term, respectively.

2. This stochastic process \eqref{adversarial_training_sde} is known as an Ornstein-Uhlenbeck process. The stationary distribution of this stochastic process is a Gaussian distribution with explicit covariance $\Sigma$. The norm of $\Sigma$ is positively correlated with $\frac{\alpha}{b}$ and norm of $B$.

3. Larger $\frac{\alpha}{b}$ and/or norm of $B$ results in smaller $\mathcal{G}$, i.e., induces tighter robust generalization bound.
\end{theorem}

\begin{proof}
Please refer to Appendix for proof.
\end{proof}

Theorem \ref{pgd_at_generalization_theorem} implies the following important statements,

(A) Diffusion term $\sqrt{\frac{\alpha}{b}}AB$ is impactful in the robust generalization, and we could manipulate diffusion to improve robust generalization.

(B) We could effectively boost robust generalization via increasing $\alpha$ and decreasing $b$. We provide extensive empirical evidence to support this claim (See e.g. Figure \ref{main_exp_fig} and Table \ref{main_exp_table}).

\subsection{Analysis of Assumptions}
\label{subsec:analysis_assump}

We first present the following standard assumptions from existing studies that are used in Theorem \ref{pgd_at_generalization_theorem} and discuss why these assumptions stand.

\begin{assumption}[\cite{Mandt17SGDApproximateBayesian,He19ControlBatch,xie2021positive} The second-order Taylor approximation]
\label{local_quadratic_assumption}
Suppose the risk function is approximately convex and 2-order differentiable, in the region close to minimum, i.e., there exists a $\delta_0>0$, such that $\mathcal{R}(\theta)= \frac{1}{2}(\theta-\theta^\ast)^TA(\theta-\theta^\ast)$ if $\lVert\theta-\theta^\ast\rVert\leq\delta_0$, where $\theta^\ast$ is a minimizer of $\mathcal{R}(\theta)$. Here $A$ is the Hessian matrix $\nabla_{\theta}^2\mathcal{R}$ around minimizer and is positive definite. Without loss of generality, we assume a minimizer of the risk is zero, i.e., $\theta^\ast=0$.
\end{assumption}

Though here we assume locally quadratic form of risk function, all our results from this study apply to locally smooth and strongly convex objectives. Note that the assumption on locally quadratic structure of loss function, even for extremely nonconvex objectives, could be justified empirically. \cite{Li2018Visualization} visualized the loss surfaces for deep structures like ResNet \cite{He16Res} and DenseNet \cite{Huang2017DenseNet}, observing quadratic geometry around local minimum in both cases. And certain network architecture designs (e.g., skip connections) could further make neural loss geometry show no noticeable nonconvexity, see e.g. Figure \ref{loss_func_vis}.

\begin{figure*}[t!]
	\centering%
 
	\begin{subfigure}{0.2\textwidth}%
		\centering%
		\includegraphics[width=3.2cm,height=3.2cm]{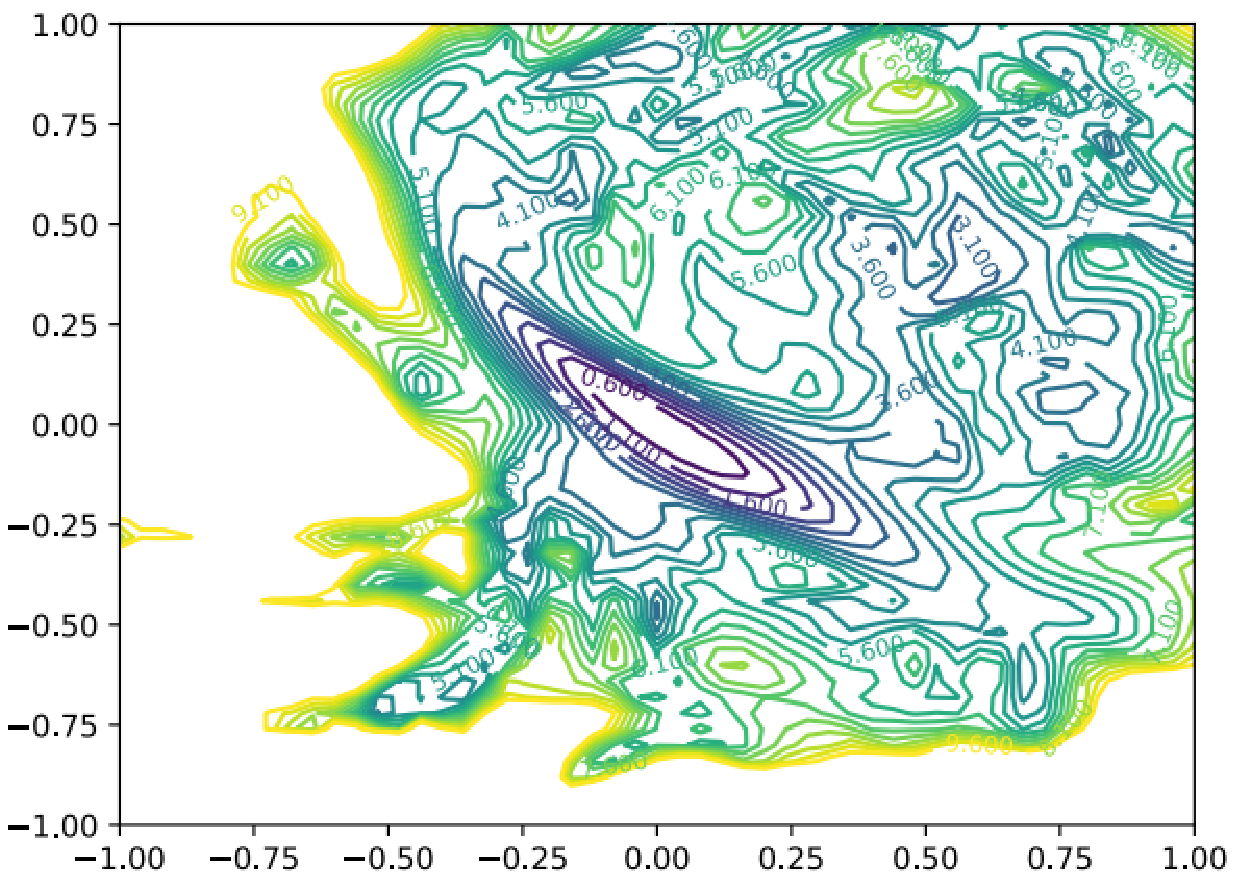}%
 
   \caption{\label{WRN56-no-shortcut-connection} WRN56-no-shortcut-connection}
	\end{subfigure}%
	\hspace{8mm}
	\begin{subfigure}{0.20\textwidth}%
		\centering%
		\includegraphics[width=3.2cm,height=3.2cm]{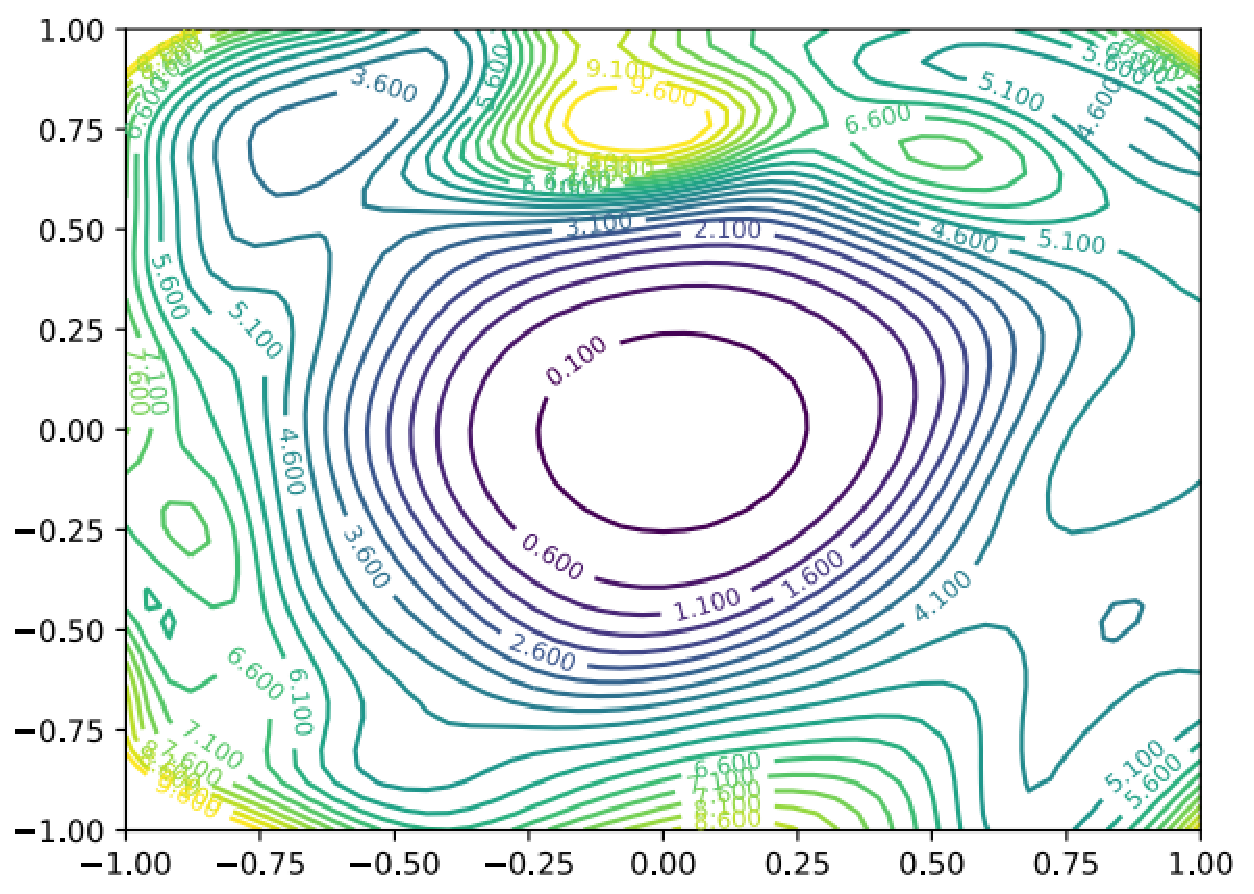}%
 
    \caption{\label{WRN56-shortcut-connection} WRN56-shortcut-connection}
	\end{subfigure}%
		\hspace{8mm}
	\begin{subfigure}{0.20\textwidth}%
		\centering%
		\includegraphics[width=3.2cm,height=3.2cm]{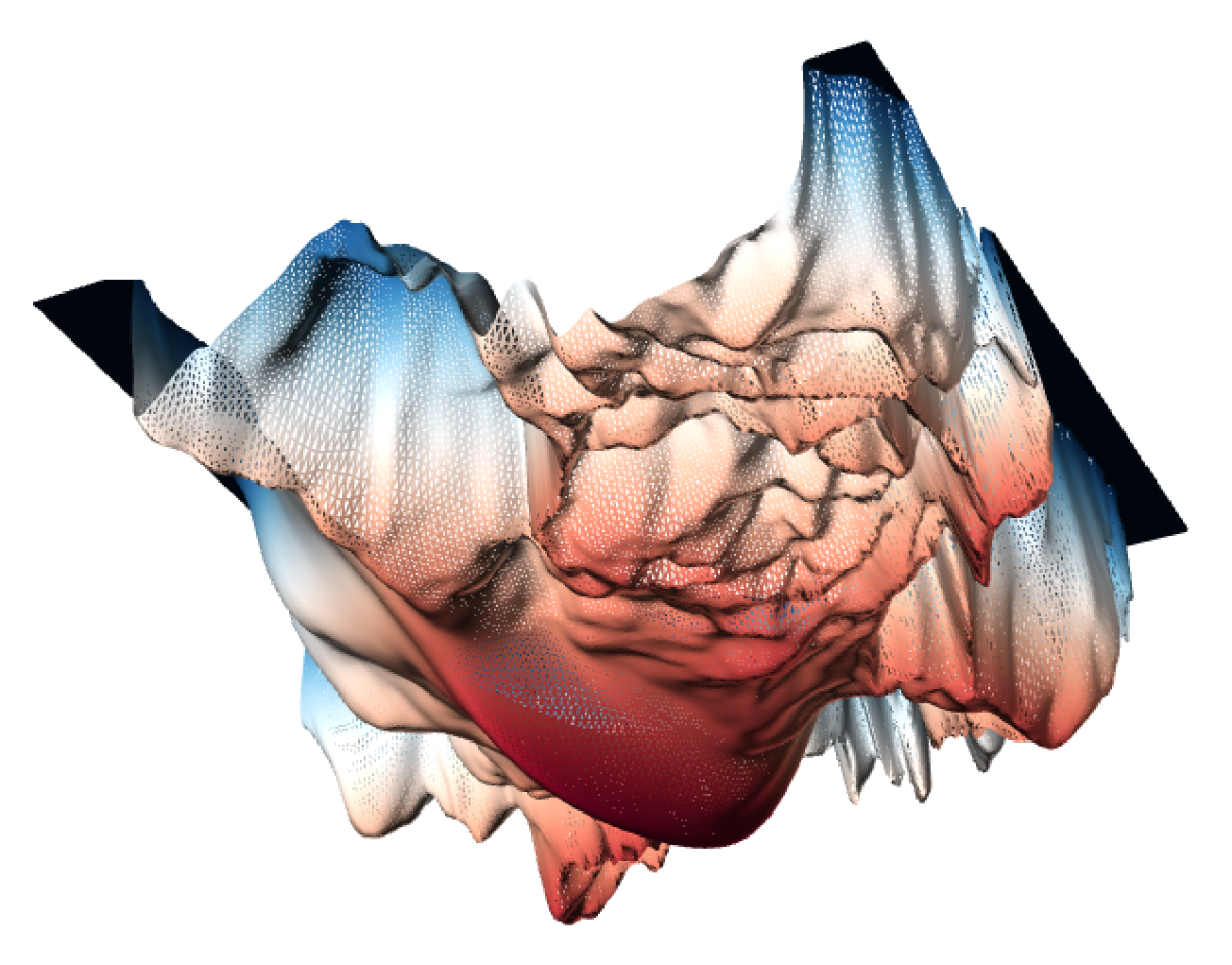}%
	 
    \caption{\label{Resnet-56-no-shortcut} Resnet-56-no-shortcut}
	\end{subfigure}%
	\hspace{8mm}
	\begin{subfigure}{0.20\textwidth}%
		\centering%
		\includegraphics[width=3.2cm,height=3.2cm]{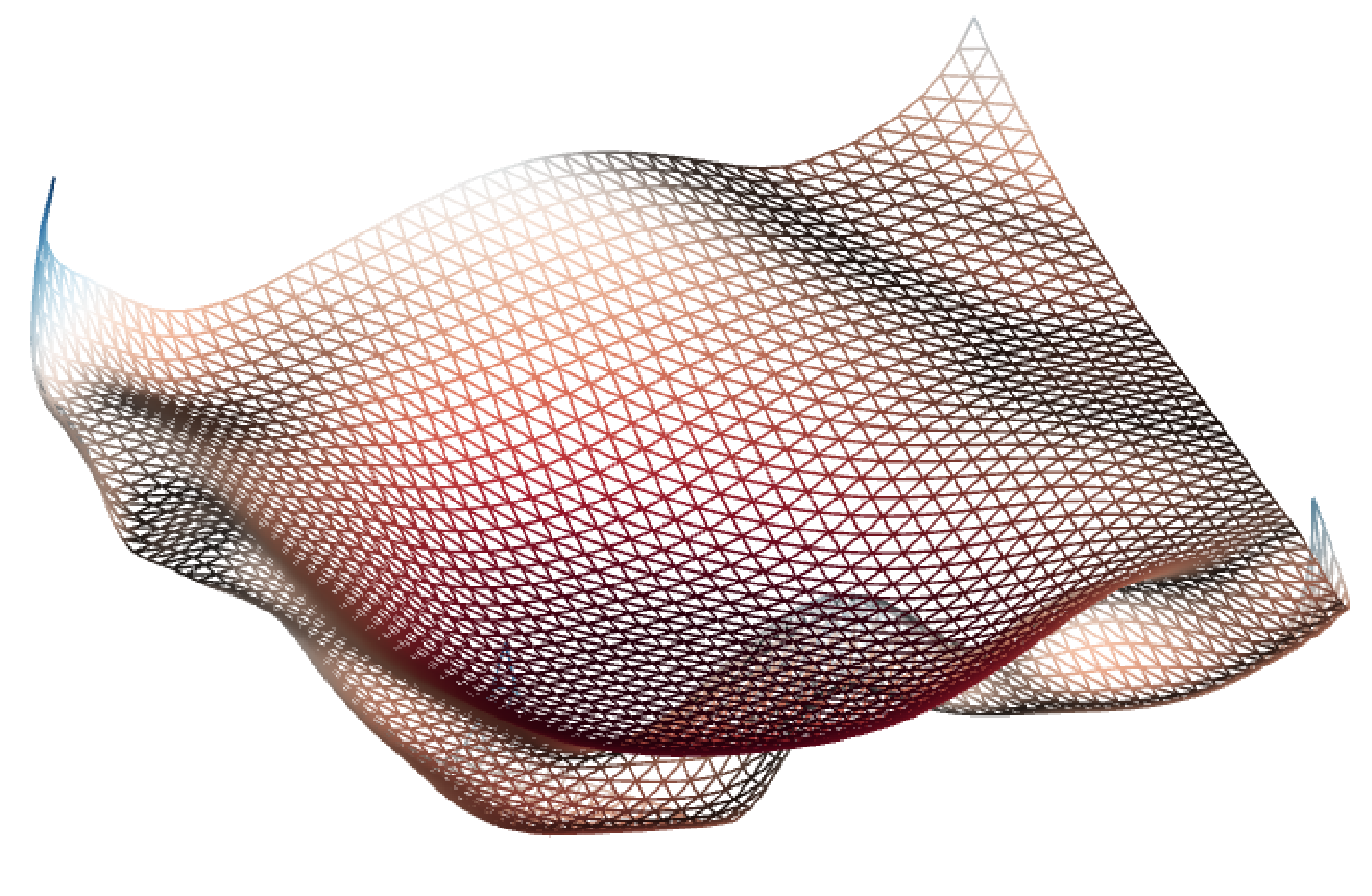}%
	  
        \caption{\label{Resnet-56-short} Resnet-56-short}
	\end{subfigure}
 
  \centering\caption{\label{loss_func_vis} 2D visualization of the loss surface of Wide-ResNet-56 on CIFAR-10 both without shortcut connections in Figure \ref{WRN56-no-shortcut-connection} and with shortcut connections in Figure \ref{WRN56-shortcut-connection} (Figure 6 in \cite{Li2018Visualization}). 3D visualization of the loss surface of ResNet-56 on CIFAR-10 both with shortcut connections in Figure \ref{Resnet-56-short} and without shortcut connections in Figure \ref{Resnet-56-no-shortcut} (from \url{http://www.telesens.co/loss-landscape-viz/viewer.html}).}

\end{figure*}


\begin{assumption}[\cite{Mandt17SGDApproximateBayesian,He19ControlBatch,xie2021positive} Unbiased Gradients with Bounded Noise Variance]
\label{gaussian_noise_assumption}
Suppose at each step $t$, gradient noise is Gaussian with mean 0 and covariance $\frac{1}{b}\Sigma(\theta_t)$, i.e., $$ \hat{g}(\theta_t)\approx g(\theta_t)+\frac{1}{\sqrt{b}}\Delta g(\theta_t), \quad \Delta g(\theta_t)\sim\mathcal{N}(0,\Sigma(\theta_t))$$ 
We further assume that the noise covariance matrix $\Sigma(\theta_t)$ is approximately constant with respect to $\theta$, i.e., $\Sigma(\theta_t)\approx\Sigma=CC^T$. And noises from different iterates $\{\Delta g(\theta_t)\}_{t\ge1}$ are mutually statistically independent.
\end{assumption}

Gaussian gradient noise is natural to assume as the stochastic gradient is a sum of $b$ independent, uniformly sampled contributions. Invoking the central limit theorem, the noise structure could be approximately Gaussian. Assumption \ref{gaussian_noise_assumption} is standard when approximating a stochastic algorithm with a continuous-time stochastic process (see e.g. \cite{Mandt17SGDApproximateBayesian}) and is justified when the iterates are confined to a restricted region around the minimizer. 

\section{Algorithm: DEAT - A 'Free' Booster to PGD-AT}
\label{algorithm_section}

Theorem \ref{pgd_at_generalization_theorem} indicates that the key factor which impacts the robust generalization is diffusion $\sqrt{\frac{\alpha}{b}}AB$. And the definitive relationship is, large diffusion level positively benefits the generalization performance of PGD-AT. 

Though increasing $\frac{\alpha}{b}$ is straightforward, there are two main drawbacks. First, decreasing batch size is impractical as it significantly lengthens training time. Adversarial training already takes notoriously lengthy time compared to standard supervised learning (as the inner maximization is essentially several steps of gradient ascent). Thus, small batch size is simply not an economical option, Second, the room to increase $\alpha$ is very limited as $\alpha$ has to be relatively small to ensure convergence.

Furthermore, we also desire an approach that could universally improve the robust generalization independent of specifications of $\alpha$ and $b$, as they could potentially complement each other to achieve a even better performance. Thus, we propose to manipulate the remaining factor in the diffusion,

\begin{center}
    \textit{Can we manipulate the gradient noise level $B$ in PGD-AT dynamic to improve its generalization?}
\end{center}

Our proposed Diffusion Enhanced AT (DEAT) (i.e. Algorithm \ref{Diffusion_PGD_AT_algorithm}) provides a positive answer to this question. The basic idea of DEAT is simple. Inspired by the idea from \cite{xie2021positive}, instead of using one single gradient estimator $\hat{g}$, Algorithm \ref{Diffusion_PGD_AT_algorithm} maintains two gradient estimators $h_t$ and $h_{t-1}$ at each iteration. A linear interpolation of these two gradient estimators is still a legitimate gradient estimator, while the noise (variance) of this new estimator is larger than any one of the base estimators. $k_1$ and $k_2$ are two hyperparameters.

We would like to emphasize when $h_t$ and $h_{t-1}$ are two unbiased and independent gradient estimators, the linear interpolation is apparently unbiased (due to linearity of expectation) and the noise of this new estimator increases. However, DEAT (and the following Theorem \ref{variance_enhanced_theorem}) does not require $h_t$ and $h_{t-1}$ to be unbiased or independent. In fact, DEAT showcases a general idea of linear combination of two estimators which goes far beyond our current design. We could certainly devise other formulation of $h_t$ or $h_{t-1}$, which may be unbiased or biased as in our current design.

It may be natural to think that why not directly inject some random noise to the gradient to improve generalization. However, existing works point out random noise does not have such appealing effect, only noise with carefully designed covariance structure and distribution class works \cite{Wu2019OnTN}. For example, \cite{zhu2019anisotropic} and \cite{daneshmand18saddle} point out, if noise covariance aligns with the Hessian of the loss surface to some extent, the noise would help generalize. Thus, \cite{Wen2019InterplayBO} proposes to inject noise using the (scaled) Fisher as covariance and \cite{zhu2019anisotropic} proposes to inject noise employing the gradient covariance of SGD as covariance, both requiring access and storage of second order Hessian which is very computationally and memory expensive.

DEAT, compared with existing literature, is the first algorithm on adversarial training, we inject noise that does not require second order information and is “free” in memory and computation.

\begin{algorithm2e}
\SetAlgoVlined
\KwIn{Loss function $J(\theta,x,\delta)=l(\theta,x+\delta,y)-\lambda R(\delta)$, initialization $\theta_0$, total training steps $T$, PGD steps $K$, inner/outer learning rates $\alpha_{I}$/$\alpha_{O}$, batch size $b$;}
\SetAlgoLined
\For{$t\in\{1,2,...,T\}$}
{
    Sample a mini-batch of random examples $\zeta=\{(x_{i_j},y_{i_j})\}_{j=1}^{b}$\;
    Set $\delta_0=0,\hat{x}_j=x_{i_j}$\;
    \For{$k\in\{1,...,K\}$}
    {
    $\delta_k=\delta_{k-1}+\frac{\alpha_{I}}{b}\sum_{j=1}^b\nabla_\delta J(\theta_{t-1},\hat{x}_j,\delta_{k-1})$\;
    }
    $h_{t}=k_2h_{t-2}+(1-k_2)\hat{g}_t$,
    $\theta_{t+1}=\theta_{t}-\alpha^\prime_{O}[(1+k_1)h_{t}-k_1h_{t-1}]$, \\where $\alpha_O^\prime=\frac{\alpha_O}{\sqrt{(1+k_1)^2+k_1^2}}$\;
    
}
return $\theta_T$
\caption{Diffusion Enhanced AT (DEAT)}
\label{Diffusion_PGD_AT_algorithm}
\end{algorithm2e}

Theorem \ref{variance_enhanced_theorem} provides a theoretical guarantee that DEAT obtains a tighter generalization bound than PGD-AT.

\begin{theorem}
\label{variance_enhanced_theorem}
Let $H_1$ and $H_2$ be the covariance matrix of gradient noise from PGD-AT and DEAT, respectively. Let $\mathcal{G}_1$ and $\mathcal{G}_2$ be the upper bounds of generalization error of PGD-AT (Algorithm \ref{PGD_AT_algorithm_revised}) and DEAT (Algorithm \ref{Diffusion_PGD_AT_algorithm}), respectively. The following statement holds,

\begin{equation}
\label{diffusion_enhanced}
    \begin{gathered}
    H_2 = kH_1, \quad \text{where} \quad k>1,\\
    \mathcal{G}_1 \ge \mathcal{G}_2
    \end{gathered}
\end{equation}

i.e., Algorithm \ref{Diffusion_PGD_AT_algorithm} generates larger gradient noise than Algorithm \ref{PGD_AT_algorithm}, and such gradient noise boosts robust generalization.
\end{theorem}

\begin{proof} We only keep primary proof procedures and omit most of the algebraic transformations. Recall the updating rule for conventional heavy ball momentum,

\begin{equation}
\label{SHB_formulation}
    \begin{gathered}
    d_{t+1}=(1-\beta)\hat{g}_t+\beta d_t\\
    \theta_{t+1}=\theta_t-\alpha d_t
\end{gathered}
\end{equation}
where $\beta$ is the momentum factor.

By some straightforward algebraic transformations, we know the momentum can be written as $d_t=(1-\beta)\sum_{i=1}^t \beta^{k-i}\hat{g}_i$.

Suppose $H$ is the noise covariance of $\hat{g}_i$ and $\nu^2$ is the scale of $H$, i.e., $\lVert H\rVert\le\nu^2$. The noise level of $d_t$ is $\approx (1-\beta)\frac{1-\beta^{t+1}}{1-\beta}\nu^2\approx\nu^2$.

Momentum $d$ does not alter the gradient noise level. We would resort to maintain two momentum terms $d^{(1)}$  $d^{(2)}$, and use the linear interpolation $(1+p)d^{(1)}-pd^{(2)}$ as our iterate.

The advantage is though the noise levels of $d^{(1)}$ and $d^{(2)}$ are both $\nu^2$, the noise level of $(1+p)d^{(1)}-pd^{(2)}$ is $\approx ((1+p)^2+p^2)\nu^2$ \cite{xie2021positive}. 

Thus, if we could show our proposed DEAT is indeed maintaining two momentum terms, we complete the proof of the statement $H_2=kH_1$ and $k>1$ in Theorem \ref{variance_enhanced_theorem}.

Recall line 7-8 in Algorithm \ref{Diffusion_PGD_AT_algorithm},

\begin{equation}
    \begin{gathered}
    h_{t}=k_2h_{t-2}+(1-k_2)\hat{g}_t,\\
    \theta_{t+1}=\theta_{t}-\alpha^\prime_{O}[(1+k_1)h_{t}-k_1h_{t-1}],
    \end{gathered}
\end{equation}

We could transform it into,

\begin{equation}
    \begin{gathered}
    h_{t}=k_2h_{t-2}+(1-k_2)\hat{g}_t,\\
    (\theta_{t+1}-\alpha^\prime_{O}k_1h_t)=(\theta_{t}-\alpha^\prime_{O}k_1h_{t-1})-\alpha^\prime_{O}h_t,
    \end{gathered}
\end{equation}

We could further write it into,

\begin{equation}
\label{pnm_rule}
    \begin{gathered}
    x_{t}=\theta_t-k_1h_{t-1},\\
    x_{t+1}=x_t-\xi \hat{g}_t+k_2(x_{t-1}-x_{t-2}),
    \end{gathered}
\end{equation}
where $\xi=\alpha^\prime_{O}(1-k_2)$. We know a conventional momentum can be written as,
\begin{equation}
\label{momentum_rewrite}
    \begin{gathered}
    \theta_{t+1}=\theta_t-\alpha\hat{g}_t+\beta(\theta_t-\theta_{t-1})
    \end{gathered}
\end{equation}
where $\alpha$ and $\beta$ are learning rate and momentum factor, respectively. Note in Equation \eqref{pnm_rule}, the second line is exactly same as in Equation \eqref{momentum_rewrite}, indicating $x_{t}$ has the same behavior as momentum. Further note $x_{t+1}-x_t=\xi h_t$, i.e., we maintain two momentum terms by alternatively using odd-number-step and even-number-step gradients. Combining everything together, we complete the proof of Theorem \ref{variance_enhanced_theorem}.
\end{proof}

One advantage of DEAT is that it adds virtually no extra parameters or computation. Though it introduces two more hyperparameters $k_1$ and $k_2$, they are highly insensitive according to our experimental investigation. 

Our experimental results in Figure \ref{main_exp_fig} and Table \ref{main_exp_table} firmly attest that DEAT outperforms PGD-AT by a significant 1.5\% to 2.0\% margin with nearly no extra burden. We would like to emphasize that 1.5\% to 2.0\% improvement with virtually no extra cost is non trivial in robust accuracy. To put 1.5\% to 2.0\% in perspective, the difference among the robust accuracy of all popular architectures is only about 2.5\% (see \cite{pang2021bag}). Our approach is nearly "free" in cost while modifying architectures includes tremendous extra parameters and model design. 2.0\% is also on par with some other techniques, e.g., label smoothing, weight decay, that are already overwhelmingly used to improve robust generalization.

Training curves in Figure \ref{adv_training_curve_fig} reveal that DEAT can beat PGD-AT in adversarial testing accuracy even when PGD-AT has better adversarial training accuracy, which shows DEAT does alleviate overfitting.

\section{Experiments}
\label{experimental_evidence}

We conduct extensive experiments to verify our theoretical findings and proposed approach. We include different architectures, and sweep across a wide range of hyperparameters, to ensure the robustness of our findings. All experiments are run on 4 NVIDIA Quadro RTX 8000 GPUs, and the total computation time for the experiments exceeds 10K GPU hours. Our code is available at \url{https://github.com/jsycsjh/DEAT}. 

We aim to answer the following two questions:

\begin{enumerate}[leftmargin=*]
    \item \textit{Do hyperparameters impact robust generalization in the same pattern as Theorem \ref{pgd_at_generalization_theorem} indicates?}
    \item \textit{Does DEAT provide a 'free' booster to robust generalization?}
\end{enumerate}

\textbf{Setup} We test on CIFAR-10 under the $l_\infty$ threat model of perturbation budget $\frac{8}{255}$, without additional data. Both the vanilla PGD-AT framework and DEAT is used to produce adversarially robust model. The model is evaluated under 10-steps PGD attack (PGD-10) \cite{madry2018towards}. Note that this paper mainly focuses on PGD attack instead of other attacks like AutoAttack \cite{Francesco20autoattack} / RayS \cite{Chen20RayS} for consistency with our theorem. The architectures we test with include VGG-19 \cite{Simonyan14VGG}, SENet-18 \cite{CVPR18SENet}, and Preact-ResNet-18 \cite{He2016IdentityMI}. Every single data point is an average of 3 independent and repeated runs under exactly same settings (i.e., every single robust accuracy in Table \ref{main_exp_table} is an average of 3 runs to avoid stochasticity). The following Table \ref{default_setting_table} summarizes the default settings in our experiments.

    \begin{table}[htbp]
 
    \caption{Experimental Settings}
    \centering
    \begin{tabular}{c|c}
    \hline
    Batch Size: 128 & Label Smoothing: False\\ 
    \hline
    Weight Decay: $5\times10^{-4}$ &  BN Mode: eval\\ 
    \hline
    Activation: ReLu &  Total Epoch: 110 \\ 
    \hline
    LR Decay Factor: 0.1 & LR Decay Epochs: 100, 105 \\ 
    \hline
    Attack: PGD-10 & Maximal Perturbation: $\epsilon=8/255$\\ 
    \hline
    Attack Step Size: $2/255$ & Threat Model: $l_\infty$\\ 
    \hline
    $k_1$: 1.0 & $k_2$: 0.8\\ 
    \hline
    \end{tabular}
    \label{default_setting_table}
 
    \end{table}
    \vspace*{-6pt}

\begin{figure}[h]
 \vspace*{-12pt}
\includegraphics[width=8.0cm]{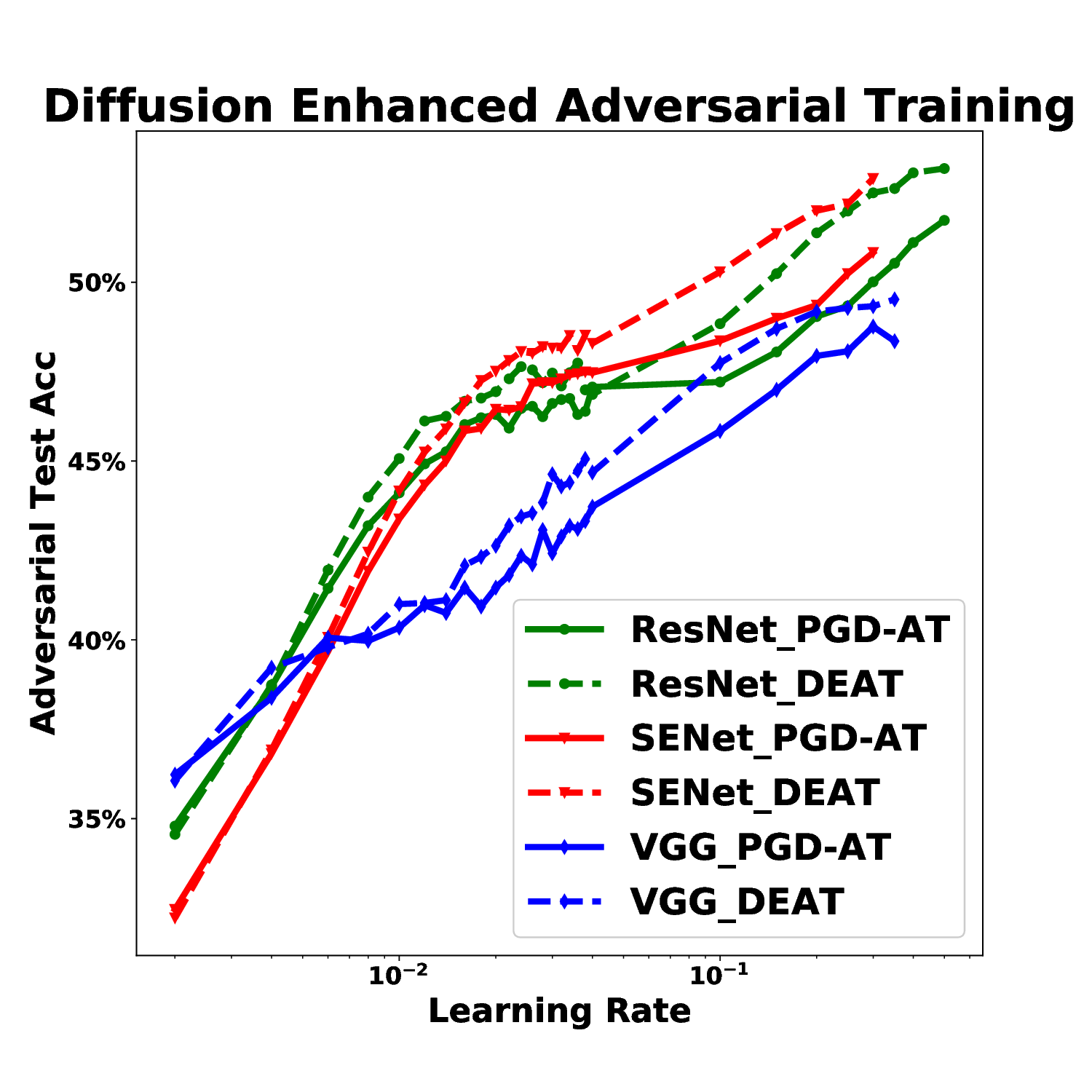}
\vspace*{-12pt}
\caption{\label{main_exp_fig} Adversarial testing accuracy on CIFAR10 for vanilla PGD-AT and our proposed DEAT across a wide spectrum of learning rates. The figure demonstrates a strongly positive correlation between robust generalization and learning rate. We could also observe DEAT obtains a significant improvement over PGD-AT.}
 
\end{figure}

Note that most of our experimental results are reported in terms of robust test accuracy, instead of the robust generalization gap. On one hand, test accuracy is the metric that we really aim to optimize in practice. On the other hand, robust test accuracy, though is not the whole picture of generalization gap, actually reflects the gap very well, especially in overparameterized regime, due to the minimization of empirical risk is relatively simple with deep models \footnote{In the setting of over-parametrized learning, there is a large set of global minima, all of which have zero training error but the test error can be very different \cite{Zhang17Rethinking, WuME18NIPS}.}, even in an adversarial environment \cite{ Rice2020OverfittingIA}. Therefore, we report only robust test accuracy following \cite{He19ControlBatch,WuME18NIPS} by default. To ensure our proposed approach actually closes the generalization gap, we report the actual generalization gap in Fig \ref{adv_training_curve_fig}, and observe DEAT can beat vanilla PGD-AT by a non-trivial margin in testing performances even with sub-optimal training performances.

\subsection{Hyperparameter is Impactful in Robust Generalization}

Our theorem indicates learning rate and batch size can impact robust generalization via affecting diffusion. Specifically, Theorem \ref{pgd_at_generalization_theorem} expects larger learning rate/batch size ratio would improve robust generalization. We sweep through a wide range of learning rates ${0.01, 0.12, 0.014, \cdots, 0.50}$, and report the adversarial testing accuracy of both vanilla PGD-AT and DEAT for a selection of learning rates in Table \ref{main_exp_table} and Figure \ref{main_exp_fig}. Considering the computational time for AT is already very long, decreasing batch size to improve robust generalization is simply economically prohibitive. Thus, we mainly focus on $\alpha$.

\newcolumntype{?}{!{\vrule width 2pt}}
\begin{table*}[ht]
\caption{Adversarial testing accuracy for both vanilla PGD-AT and DEAT. $\text{Acc}_{\text{d}}$ represents the accuracy difference between diffusion enhanced adversarial training and vanilla PGD-AT, i.e., $\text{Acc}_{\text{DEAT}}-\text{Acc}_{\text{PGD-AT}}$.}
 
\centering
\begin{tabular}{|c|c|c|c?c|c|c|c?c|c|c|c|}
\multicolumn{4}{c?}{Preact-ResNet \cite{He16Res}}&\multicolumn{4}{c?}{SENet \cite{CVPR18SENet}} &\multicolumn{4}{c}{VGG \cite{Simonyan14VGG}}\\
    \hline
    $\alpha$ & PGD-AT & DEAT & $\text{Acc}_{\text{d}}$ & $\alpha$ & PGD-AT & DEAT & $\text{Acc}_{\text{d}}$ & $\alpha$ & PGD-AT & DEAT & $\text{Acc}_{\text{d}}$\\\hline
    0.010 & 44.11\% & 45.07\% & 0.96\% & 0.010 & 43.38\% & 44.16\% & 0.78\% & 0.010 & 40.34\% & 41.00\% & 0.66\%\\
    0.012 & 44.92\% & 46.12\% & 1.20\% & 0.012 & 44.33\% & 45.25\% & 0.92\% & 0.012 & 40.97\% & 41.03\% & 0.06\%\\
    0.014 & 45.26\% & 46.25\% & 0.99\% & 0.014 & 45.00\% & 45.90\% & 0.90\% & 0.014 & 40.75\% & 41.11\% & 0.36\%\\
    0.018 & 46.21\% & 46.76\% & 0.55\% & 0.018 & 45.91\% & 47.25\% & 1.34\% & 0.018 & 40.93\% & 42.32\% & 1.39\%\\
    0.020 & 46.30\% & 46.94\% & 0.64\% & 0.020 & 46.45\% & 47.51\% & 1.06\% & 0.020 & 41.46\% & 42.08\% & 0.62\%\\
    0.022 & 45.92\% & 47.30\% & 1.38\% & 0.022 & 46.42\% & 47.81\% & 1.39\% & 0.022 & 41.81\% & 43.20\% & 1.39\%\\
    0.024 & 46.47\% & 47.64\% & 1.17\% & 0.024 & 46.52\% & 48.06\% & 1.54\% & 0.024 & 42.35\% & 43.45\% & 1.10\%\\
    0.028 & 46.24\% & 47.19\% & 0.95\% & 0.028 & 47.19\% & 48.20\% & 1.01\% & 0.028 & 43.07\% & 43.84\% & 0.77\%\\
    0.030 & 46.61\% & 47.46\% & 0.85\% & 0.030 & 47.19\% & 48.16\% & 0.97\% & 0.030 & 42.42\% & 44.63\% & 2.21\%\\
    \hline
    0.100 & 47.21\% & 48.84\% & 1.63\% & 0.100 & 48.36\% & 50.29\% & 1.93\% & 0.100 & 45.84\% & 47.74\% & 1.90\%\\
    0.150 & 48.05\% & 50.24\% & 2.19\% & 0.150 & 48.99\% & 51.36\% & 2.37\% & 0.150 & 46.99\% & 48.70\% & 1.71\%\\
    0.200 & 49.04\% & 51.38\% & 2.34\% & 0.200 & 49.36\% & 52.00\% & 2.64\% & 0.200 & 47.94\% & 49.18\% & 1.24\%\\
    0.250 & 49.34\% & 51.99\% & 2.65\% & 0.250 & 50.24\% & 52.19\% & 1.95\% & 0.250 & 48.07\% & 49.28\% & 1.21\%\\
    0.300 & 50.01\% & 52.50\% & 2.49\% & 0.300 & 50.83\% & 52.90\% & 2.07\% & 0.300 & 48.76\% & 49.33\% & 0.57\%\\
    \specialrule{0.2em}{0em}{0.05em} 
    \multicolumn{12}{|c|}{Improvement is significant especially when model is performing well (large $\alpha$).} \\
    \multicolumn{12}{|c|}{\textbf{DEAT improves \underline{1.5\%} on VGG, and over \underline{2.0\%} on SENet and Preact-ResNet.}}\\
    \specialrule{0.2em}{0em}{0.05em} 
\end{tabular}
\label{main_exp_table}
 
\end{table*}

Table \ref{main_exp_table} exhibits a strong positive correlation between robust generalization and learning rate. The pattern is consistent with all three architectures. Figure \ref{main_exp_fig} provides a better visualization of the positive correlation. 

We further do some testing on whether such correlation is statistically significant or not. We calculate the Pearson's $r$, Spearman's $\rho$, and Kendall's $\tau$ rank-order correlation coefficients \footnote{They measure the statistical dependence between the rankings of two variables, and how well the relationship between two variables can be described using a monotonic function.}, and the corresponding $p$ values to investigate the statistically significance of the correlations. The procedure to calculate $p$ values is as follows, when calculating $p$-value in Tables \ref{rank_correlation_table} and \ref{t_test_table}, we regard the data point in Table \ref{main_exp_table} as the accuracy for each $\alpha$ and calculate the RCC between accuracy and $\alpha$ and its $p$-value, following same procedure in \cite{He19ControlBatch}.

We report the test result in Table \ref{rank_correlation_table}. The closer correlation coefficient is to $+1$ (or $-1$), the stronger positive (or negative) correlation exists. If $p<0.005$, the correlation is statistically significant \footnote{The criterion of 'statistically significant' has various versions, such as $p < 0.05$ or $p < 0.01$. We use a more rigorous $p < 0.005$.}.

\newcolumntype{?}{!{\vrule width 2pt}}
\begin{table*}[ht]
\caption{Rank correlation coefficients (corresponding significance level) between robust generalization and learning rate. All correlation coefficient indicates a strong positive relationship (close to $+1$). The p values are all highly statistically significant.}
\centering
\begin{tabular}{c?c|c?c|c?c|c}
\multicolumn{1}{c?}{Rank Correlation Coefficient} & 
\multicolumn{2}{c?}{
\begin{tabular}{c}
      Preact-ResNet  \\
      \hline
      \begin{tabular}{c|c}
      PGD-AT & DEAT
 \end{tabular}
 \end{tabular}
 }&
 \multicolumn{2}{c?}{
\begin{tabular}{c}
      SENet  \\
      \hline
      \begin{tabular}{c|c}
      PGD-AT & DEAT
 \end{tabular}
 \end{tabular}
 } &\multicolumn{2}{c}{
\begin{tabular}{c}
      VGG  \\
      \hline
      \begin{tabular}{c|c}
      PGD-AT & DEAT
 \end{tabular}
 \end{tabular}
 }\\
    \hline
    
    Pearson's $r$ ($p$-value) & 0.889 (\textbf{5.5e-10}) & 0.896 (\textbf{2.8e-10}) & 0.711 (\textbf{1.4e-04}) & 0.762 (\textbf{2.4e-05}) & 0.916 (\textbf{3.3e-10}) & 0.862 (\textbf{5.8e-08})\\
    Spearman's $\rho$ ($p$-value) & 0.965 (\textbf{3.4e-16}) & 0.922 (\textbf{7.5e-07}) & 0.998 (\textbf{<2.2e-16}) & 0.982 (\textbf{<2.2e-16}) & 0.988 (\textbf{<2.2e-16}) & 0.992 (\textbf{8.9e-07})\\
    Kendall's $\tau$ ($p$-value) & 0.907 (\textbf{3.3e-11}) & 0.818 (\textbf{1.9e-12}) & 0.982 (\textbf{5.7e-11}) & 0.927 (\textbf{6.3e-10}) & 0.932 (\textbf{1.8e-10}) & 0.956 (\textbf{<2.2e-16})\\
    \hline
\end{tabular}
\label{rank_correlation_table}
\end{table*}

Our theorem indicates ratio of learning rate and batch size (instead of batch size itself) determines generalization, which justifies the linear scaling rule in \cite{pang2021bag}, i.e., scaling the learning rate up when using larger batch, and maintaining the ratio between learning rate and batch size, would effectively preserve the robust generalization.

The side effect of adjusting batch size also demonstrates the necessity of our proposed approach, which could manipulate diffusion to boost generalization without extra computational burden.

\begin{figure*}[t!]
	\centering%
 
	\begin{subfigure}{0.30\textwidth}%
		\centering%
		\includegraphics[width=5.8cm,height=5.8cm]{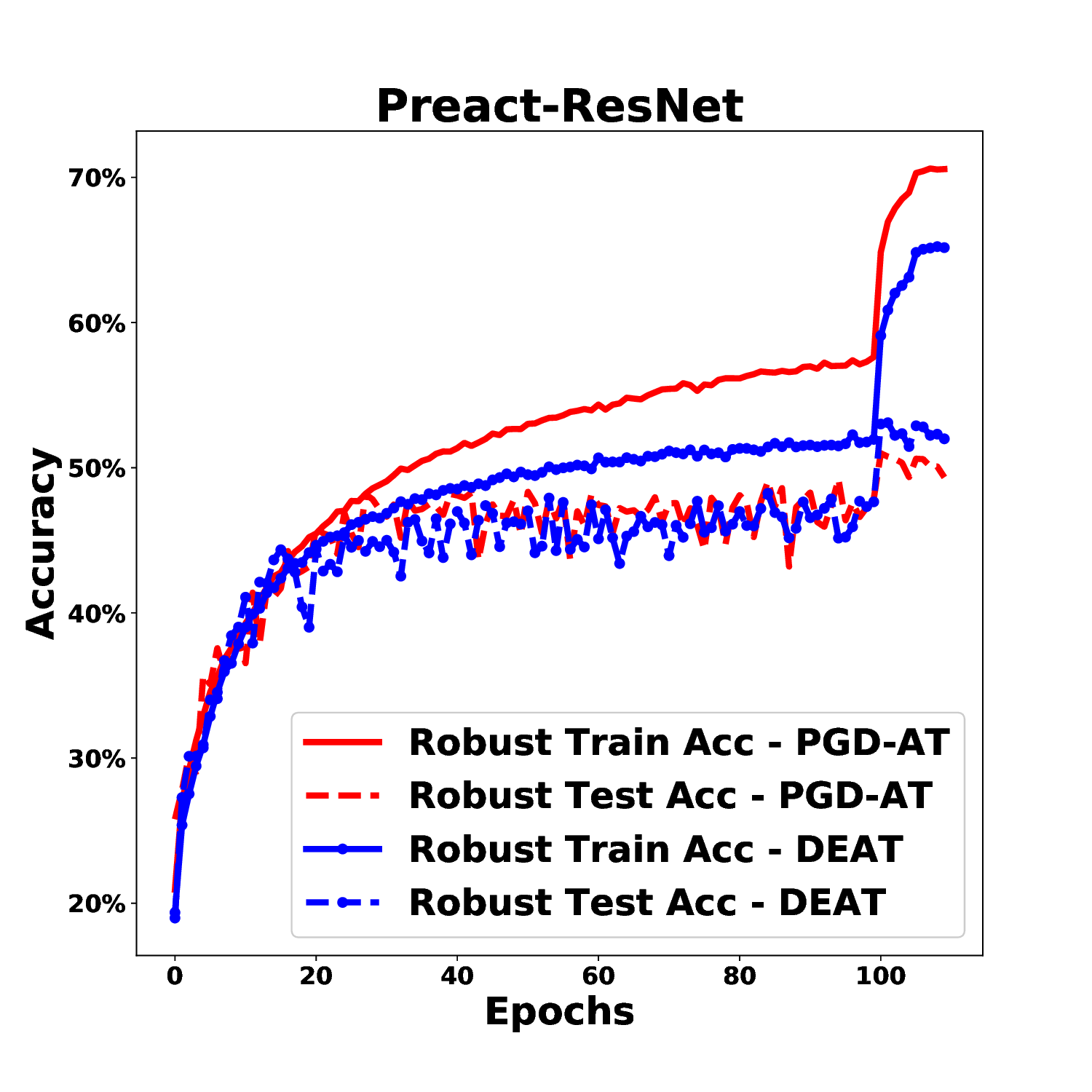}%
 
	\end{subfigure}%
	\hspace{8mm}
	\begin{subfigure}{0.30\textwidth}%
		\centering%
		\includegraphics[width=5.8cm,height=5.8cm]{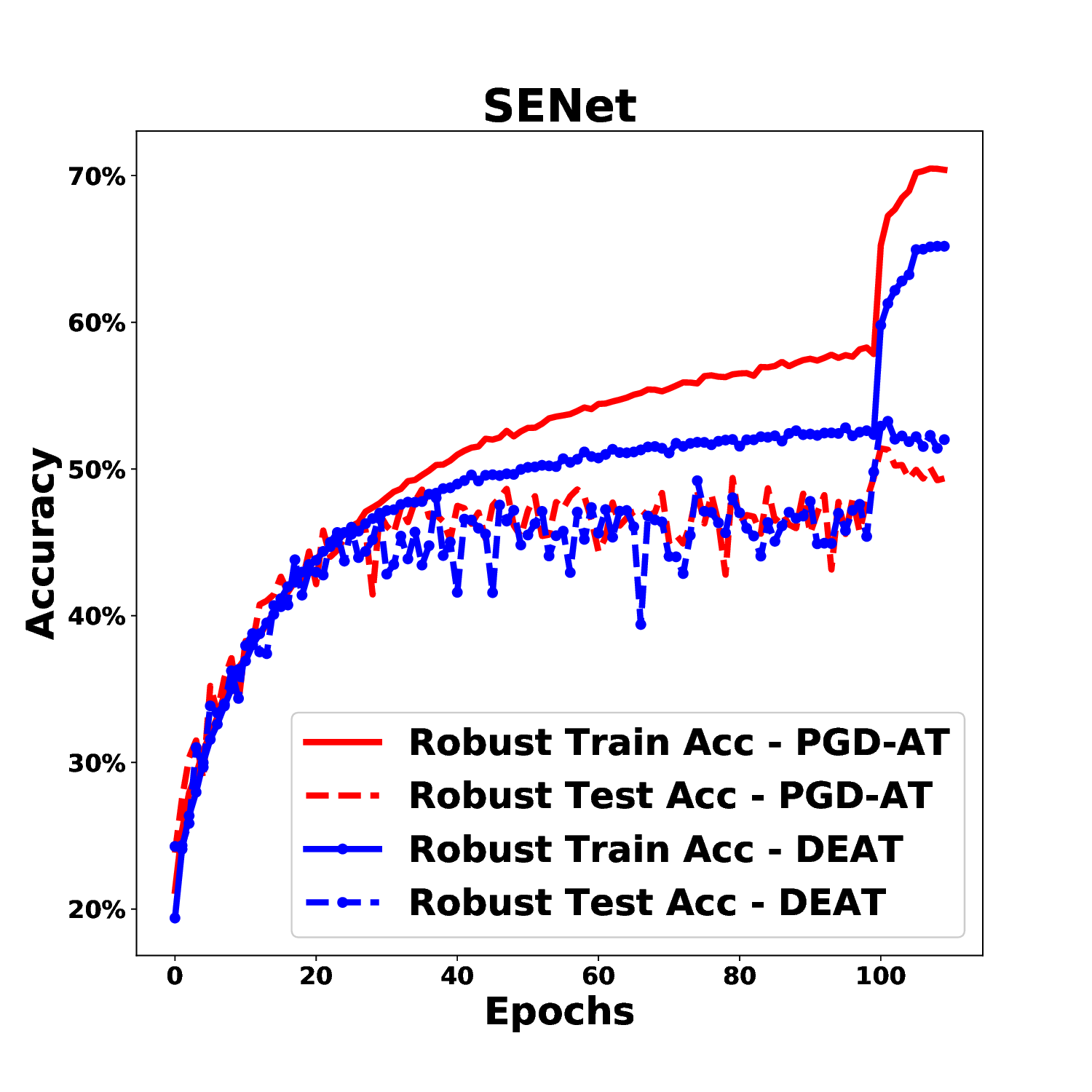}%
  
	\end{subfigure}%
		\hspace{8mm}
	\begin{subfigure}{0.30\textwidth}%
		\centering%
		\includegraphics[width=5.8cm,height=5.8cm]{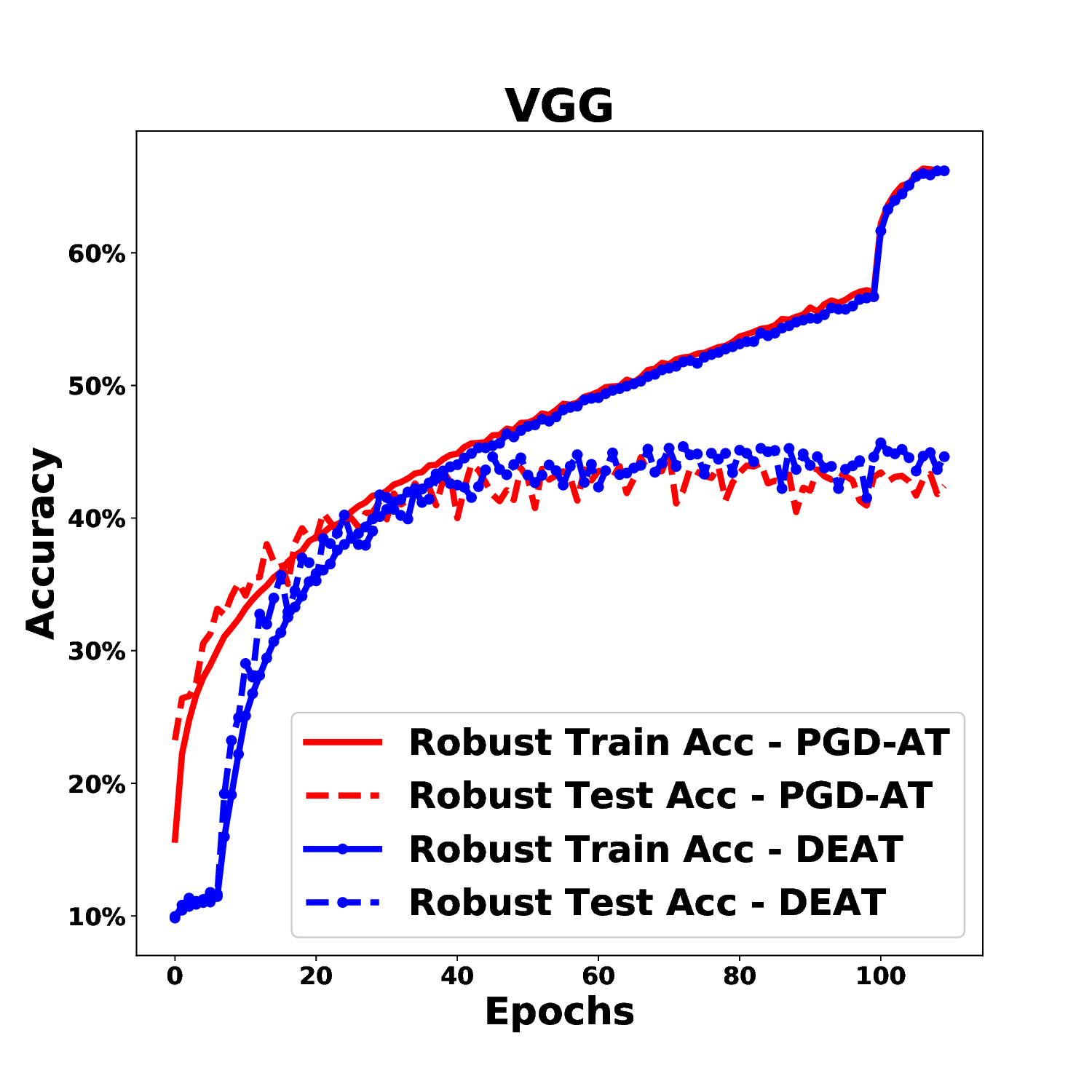}%
 
	\end{subfigure}%
 
  \centering\caption{\label{adv_training_curve_fig} Adversarial training and adversarial testing curves for vanilla PGD-AT and DEAT. DEAT performs worse in training stage, but outperforms vanilla PGD-AT in testing stage. This pattern strongly attests to the effectiveness of DEAT in alleviating overfitting.}
\end{figure*}

\subsection{DEAT Effectively Improves Robust Generalization}

We compare the robust generalization of vanilla PGD-AT and DEAT in Figure \ref{main_exp_fig} and Table \ref{main_exp_table}. 

The improvement is consistent across all different learning rates/model architectures. The improvement is even more significant when learning rate is fairly large, i.e. when the baseline is working well, in both Table \ref{main_exp_table} and Figure \ref{main_exp_fig}. Our proposed DEAT improves 1.5\% on VGG, and over 2.0\% on SENet and Preact-ResNet.

Note 1.5\% to 2.0\% improvement is very significant in robust generalization. It actually surpasses the performance gap between different model architectures. In Figure \ref{main_exp_fig}, the boosted VGG can obtain similar robust generalization compared to SENet and ResNet. \cite{pang2021bag} measures the robust generalization of virtually all popular architectures, and the range is only approximately 3\%. Considering adjusting architectures would potentially include millions of more parameters and carefully hand-crafted design, our proposed approach is nearly "free" in cost.

We plot the adversarial training and adversarial testing curves (using one specific learning rate) for all three architectures in Figure \ref{adv_training_curve_fig}. It is very interesting to observe that our proposed approach may not be better in terms of training performances (e.g. in ResNet and SENet), but it beats vanilla PGD-AT by a non-trivial margin in testing performances. It is safe to say that DEAT effectively control the level of overfitting in adversarial training.

We further do a t-test to check the statistical significance of the improvement and report the result in Table \ref{t_test_table}. Note the mean improvement in the table (e.g. 1.22\%) is averaged across all learning rates, and does not completely reflect the extent of improvement (as we pay more attention to the improvement with larger learning rates, where the improvement is larger than 1.5\%). The p-values clearly indicate a statistical significant improvement across models.

\begin{table}[htbp]
\caption{Statistical test of significance of improvement. The p-values indicate a strongly significant improvement across all architectures.}
\centering
\begin{tabular}{c|c}
\hline
Architecture & Statistical Significance of Improvement \\
\hline
Preact-ResNet & 1.22\% (\textbf{2.10e-09}) \\ 
\hline
SENet & 1.21\% (\textbf{2.11e-09}) \\ 
\hline
VGG & 1.11\% (\textbf{7.462e-10}) \\ 
\hline
\end{tabular}
\label{t_test_table}
\end{table}

\section{Conclusions}
To our best knowledge, this paper is the first study that rigorously connects the dynamics of adversarial training to the robust generalization. Specifically, we derive a generalization bound of PGD-AT, and based on this bound, point out the role of learning rate and batch size. We further propose a novel training approach Diffusion Enhanced Adversarial Training. Our extensive experiments demonstrate DEAT universally outperforms PGD-AT by a large margin with little cost, and could potentially serve as a new strong baseline in AT research.



\bibliographystyle{ACM-Reference-Format}
\bibliography{jianhui}

\newpage
\appendix

\section{Appendix}
\label{appendix_section}

\subsection{Related Work}
\label{related_work}

We summarize the related works in (A) Adversarial Training; (B) SDE Modeling of Stochastic Algorithms; and (C) Generalization and Hyperparameters in this section.

\subsubsection{Adversarial Training}

Deep learning models have been widely applied in many different domains, e.g. vision \cite{He2016DeepResNet,Huang2017DenseNet,dosovitskiy2020vit}, text \cite{Devlin2019BERT,Raffel2019ExploringTL,brown2020language}, graph \cite{Jure2017GNN,kipf2016gcn,DingZhou21WWW}, recommender system \cite{google16deep&wide,Wei22Recommender}, healthcare \cite{SuoICHI19,Xun2020CorrelationNF}, sports \cite{Mnih2013PlayingAW,Chen22CIKMSports,Chen2023ProfessionalBP}, while they are observed to be susceptible to human imperceptible adversarial attacks \cite{Szegedy2014IntriguingPO,Goodfellow2015ExplainingAH,nguyen2015adversarial,zugner18adversarialgraph,Huai20attack,Sinha2022UnderstandingAE,Zhou23TKDE}. We refer readers to a comprehensive overview of adversarial attacks and defenses and references therein \cite{Chakraborty2018AdversarialAA}. An incomplete list of recent advances would include \cite{tramer2018ensemble, Zhang2019TheoreticallyPT, pmlr-v97-wang19i, Qin2019AdversarialRT, wang2018a, Mao2019MetricLF, Carmon2019UnlabeledDI, hendrycks2019pretraining, wu22adversarial,Wu2022AdversarialWP, Shafahi2019AdversarialTF, zhang2019you, Wong2020Fast}. This study focuses on PGD-AT \cite{madry2018towards}, the most commonly used adversarial training strategy, and we view PGD-AT as a minimax optimization problem \cite{lin20minimaxa,lin20minimaxb}, where the inner maximization optimizes an adversary while the outer minimization robustifies the model parameters.

\subsubsection{SDE Modeling of Stochastic Algorithms}

Studying training dynamics is a very important perspective to probe the inner mechanism of deep learning training \cite{Li17SME,Mandt17SGDApproximateBayesian,Wei23ICML}. \cite{Li17SME,Mandt17SGDApproximateBayesian} are the first works that approximate discrete-time stochastic gradient descent by continuous-time SDE. \cite{Krichene17Mirror,An2018StochasticME,Gitman19Momentum,Cao2020ApproximationAC} extended SDE modeling to accelerated mirror descent, asynchronous SGD, momentum SGD, and generative adversarial networks, respectively. \cite{liu2021noise} studied the SDE approximation of SGD with a moderately large learning rate, while approximation in \cite{Mandt17SGDApproximateBayesian} works best with infinitesimal step size. \cite{Chaudhari2017DeepRP,xie2021positive} designed an entropy regularization and a noise injection method, respectively, motivated by the SDE characterization of SGD. \cite{Gu2021AT_SDE} attempted to model adversarial training dynamics via SDE, while did not recognize the connection between dynamics and generalization error.

\subsubsection{Generalization and Stochastic Noise}

One of the goals of this paper is to theoretically and empirically study how stochastic noise impacts generalization in adversarial training. The research is mainly divided into two lines, the impact of hyperparameters on noise and directly injection of external noise. 

Existing works on hyperparameters are mainly on non-adversarial training, e.g., many recent works empirically report the influence of hyperparameters in SGD, largely on $b$ and $\alpha$, and provide practical tuning guidelines. \cite{Keskar16Large-Batch} empirically showed that the large-batch training leads to sharp local minima which has poor generalization, while small-batches lead to flat minima which makes SGD generalize well. \cite{GoyalDGNWKTJH17LargeMinibatch,jastrzebski2018three} proposed the Linear Scaling Rule for adjusting $\alpha$ as a function of $b$ to maintain the generalization ability of SGD. \cite{Smith18Bayesian,Smith18DontDecay} suggested that increasing $b$ during the training can achieve similar result of decaying the learning rate $\alpha$. 

Our generalization analysis relies on PAC-Bayesian inequalities \cite{Williamson97PAC-Bayes,McAllester98PAC-Bayes,Guedj2019APO}. \cite{London2017APA,He19ControlBatch,Sun21KDDHyperparameter,Sun22KDDHyper,Sun2023TKDD} proved a PAC-Bayesian bound for vanilla SGD, SGD with momentum, asynchronous SGD, all in a benign environment. 

The first and only systematic study on hyperparameters of adversarial training is \cite{pang2021bag}, to our best knowledge. The authors carefully evaluated a wide range of training tricks, including early stopping, learning rate schedule, activation function, model architecture, optimizer and many others. However, their findings do not provide theoretical insights why certain tricks work or fail. Our study aims to bridge this gap and motivate our novel training algorithm through theoretical findings.

\subsection{Proof of Theorem \ref{pgd_at_generalization_theorem}}

In this section, we give the proof of Theorem \ref{pgd_at_generalization_theorem}. We only keep primary proof procedures and omit most of the algebraic transformations.

\subsubsection{Pseudocode of PGD-AT}

Constrained optimization is typically transformed into unconstrained optimization in real-world deployment. After adding a regularization term $R(\delta)$ in \eqref{adversarial_training_task_revised}, the constrained inner optimization $\delta_k=\Pi_{\Delta}(\delta_{k-1}+\frac{\alpha_{I}}{b}\sum_{j=1}^b\nabla_x l_{\theta_{t-1}}(\hat{x}_j+\delta_{k-1},y_{i_j}))$ (constrained by perturbation budget set $\Delta$) is transformed into an unconstrained optimization, with $\lambda$ a hyperparameter. As perturbation budget set $\Delta$ is typically in the form $\{\delta\in\mathbb{R}^d:\lVert\delta\rVert_p\le \epsilon\}$, i.e., the $L_p$ norm of perturbation is smaller than a constant $\epsilon$, objective function \eqref{adversarial_training_task_revised} is simply a Lagrange relaxation of \eqref{adversarial_training_task} with $R(\delta_i)=\lVert\delta\rVert_p-\epsilon$.

\begin{equation}
\label{adversarial_training_task_revised}
    \begin{gathered}
    \min_{\theta\in\mathbb{R}^{d_\theta}}\max_{\delta_i,i=1,\dots,N}\frac{1}{N}\sum_{i=1}^N J(\theta,x_i,\delta_i), \\
    \text{where} \quad J(\theta,x_i,\delta_i) = l(\theta,x_i+\delta_i,y_i)-\lambda R(\delta_i)
    \end{gathered}
\end{equation}

With the above modification, the following modified PGD-AT is also widely used (i.e., Algorithm \ref{PGD_AT_algorithm_revised}).

\begin{algorithm2e}
\SetAlgoVlined
\KwIn{Loss function $J(\theta,x,\delta)=l(\theta,x+\delta,y)-\lambda R(\delta)$, initialization $\theta_0$, total training steps $T$, PGD steps $K$, inner/outer learning rates $\alpha_{I}$/$\alpha_{O}$, batch size $b$;}
\SetAlgoLined
\For{$t\in\{1,2,...,T\}$}
{
    Sample a mini-batch of random examples $\zeta=\{(x_{i_j},y_{i_j})\}_{j=1}^{b}$\;
    Set $\delta_0=0,\hat{x}_j=x_{i_j}$\;
    \For{$k\in\{1,...,K\}$}
    {
    $\delta_k=\delta_{k-1}+\frac{\alpha_{I}}{b}\sum_{j=1}^b\nabla_\delta J(\theta_{t-1},\hat{x}_j,\delta_{k-1})$\;
    }
    $\theta_t=\theta_{t-1}-\alpha_{O}\hat{g}_{t-1}$, \\where $\hat{g}_{t-1}=\frac{1}{b}\sum_{j=1}^b\nabla_\theta J(\theta_{t-1},\hat{x}_j,\delta_K)$\;
}
return $\theta_T$
\caption{Modified PGD-AT}
\label{PGD_AT_algorithm_revised}
\end{algorithm2e}

\subsubsection{Proof of Theorem \ref{pgd_at_generalization_theorem}}

The representation of PGD-AT in SDE form is followed from \cite{Gu2021AT_SDE}. Recall PGD-AT is a min-max game \cite{huang2021efficient,zhang2022revisiting,wu2023decentralized} as in Algorithm \ref{PGD_AT_algorithm_revised}. Without loss of generality, we assume $\alpha_O=\alpha_I=\alpha$. The inner loop starts with a random initialization $\delta_0$, we have 
\begin{equation}
    \begin{gathered}
    \delta_1=\delta_0+\frac{\alpha}{b}\sum_{j=1}^b\nabla_\delta J(\theta_t,\hat{x}_j,0)=\frac{\alpha}{b}\sum_{j=1}^b\nabla_x l(\theta_t,\hat{x}_j)\\
    \delta_2=\delta_1+\frac{\alpha}{b}\sum_{j=1}^b\Big(\nabla_x l(\theta_t,\hat{x}_j+\delta_1)-\lambda \nabla_\delta R(\delta_1)\Big)
\end{gathered}
\end{equation}

Based on Taylor's expansion at $\delta=0$, we could get:

\begin{equation}
    \begin{gathered}
    \delta_2=\frac{2\alpha}{b}\sum_{j=1}^b\nabla_x l(\theta_t,\hat{x}_j) + O(\alpha^2)
\end{gathered}
\end{equation}

We could continue this calculation, we will see that for any $K$,

\begin{equation}
    \begin{gathered}
    \delta_K=\frac{K\alpha}{b}\sum_{j=1}^b\nabla_x l(\theta_t,\hat{x}_j) + O(\alpha^2)
\end{gathered}
\end{equation}

We only keep the $O(\alpha)$ term and higher order term is negligible. The outer loop updating dynamic will subsequently become,

\begin{equation}
    \begin{gathered}
        \theta_{t+1}=\theta_{t}-\frac{\alpha}{b}\sum_{j=1}^b\nabla_\theta J(\theta_{t},\hat{x}_j,\delta_K)\\
        =\theta_{t}-\frac{\alpha}{b}\sum_{j=1}^b\nabla_\theta l(\theta_t,\hat{x}_j+\frac{K\alpha}{b}\sum_{j=1}^b\nabla_x l(\theta_t,\hat{x}_j) + O(\alpha^2))\\
        =\theta_{t}-\frac{\alpha}{b}\sum_{j=1}^b\nabla_\theta l(\theta_t,\hat{x}_j)\\-\frac{K\alpha^2}{b^2}\sum_{i,j=1}^b \nabla_{x\theta}l(\theta_t,\hat{x}_j)\nabla_{x}l(\theta_t,\hat{x}_i)+O(\alpha^3)
    \end{gathered}
\end{equation}
where $\nabla_{x\theta}l$ is the Jacobian matrix, where $\frac{\partial^2 l}{\partial x_j \partial \theta_i}$ is calculated.

We compute the first and second order moments of the one-step difference $D=\theta_1-\theta_0$ in order to find a continuous-time SDE for PGD-AT, i.e., calculate $\mathbb{E}[D]$ and $\mathbb{E}[DD^T]$, where the expectation is taken with respect to the randomness of mini-batch sampling.

By some algebraic transformations, we are able to show the following results after omitting higher order term of $O(\alpha^3)$ \cite{Gu2021AT_SDE},

\begin{equation}
    \begin{gathered}
        \mathbb{E}[D]=-K\alpha^2\mathbb{E}[\nabla_{x\theta}l(\theta_0,x)]\mathbb{E}[\nabla_x l(\theta_0,x)]\\-\alpha\mathbb{E}[\nabla_\theta l(\theta_0,x)]
        -\frac{K\alpha^2}{b}\Big(\mathbb{E}[\nabla_{x\theta}l(\theta_0,x)\nabla_x l(\theta_0,x)]\\-\mathbb{E}[\nabla_{x\theta}l(\theta_0,x)]\mathbb{E}[\nabla_x l(\theta_0,x)]\Big)\\
        \mathbb{E}[DD^T]=\alpha^2\mathbb{E}[\nabla_\theta l(\theta_0,x)]\mathbb{E}[\nabla_\theta l(\theta_0,x)]^T\\
        +\frac{\alpha^2}{b}\text{Var}_x(\nabla_\theta l(\theta_0,x))
    \end{gathered}
\end{equation}

Let us notate,
\begin{equation}
    \begin{gathered}
        G(\theta) \triangleq g(\theta)+\frac{K\alpha}{2b}\Big(\mathbb{E}[\lVert \nabla_x l(\theta,x)\rVert^2]-\lVert D(\theta)\rVert^2 \Big)
    \end{gathered}
\end{equation}
for the sake of convenience. Here $g(\theta)\triangleq\mathbb{E}[l(\theta,x)]$, and $D(\theta)\triangleq\mathbb{E}[\nabla_x l(\theta,x)]$.

With $\mathbb{E}[D]$ and $\mathbb{E}[DD^T]$, we are able to show the following SDE could approximate the continuous-time dynamic of PGD-AT,

\begin{equation}
    \begin{gathered}
        d\theta=\Big(s_0+\alpha s_1 \Big)dt+\sigma dW_t
    \end{gathered}
\end{equation}
where $dW_t=\mathcal{N}(0,Idt)$ is a Wiener process. And $s_0$, $s_1$, and $\sigma$ are $-\nabla_\theta G(\theta)$, $-\frac{K}{2}\nabla_\theta(\lVert D(\theta)\rVert^2)-\frac{1}{4}\nabla_\theta (\lVert\nabla_\theta G(\theta)\rVert^2)$, and $\sqrt{\frac{\alpha}{b}}(\text{Var}_x(\nabla_\theta l(\theta,x)))^{\frac{1}{2}}$, respectively.

As we assume the risk function is locally quadratic, and gradient noise is Gaussian. Suppose Hessian matrix of risk function be $A$, and covariance matrix of Gaussian noise be $H=BB^T$. Consider the following second-order Taylor approximation of the loss function, $l(\theta,x)=\frac{1}{2}(\theta-x)^TA(\theta-x)-\text{Tr}(A)$. Gradient noise is equivalent to assuming $x$ is sampled from $\mathcal{N}(0,H)$ (without loss of generality, we could assume the data is centered at mean 0).

We could thus computing the following,

\begin{equation}
    \begin{gathered}
        g(\theta)\triangleq\mathbb{E}_x[l(\theta,x)]=\frac{1}{2}\theta^TA\theta\\
        \nabla_x l(\theta,x) = -\nabla_\theta l(\theta,x) = A(x-\theta),\qquad \nabla_{x\theta}l(\theta,x)=A
    \end{gathered}
\end{equation}

Thus, we consequently have, the drift term is $s_0+\alpha s_1=-(A+(K+\frac{1}{2})\alpha A^2)\theta$, and diffusion term $\sqrt{\frac{\alpha}{b}}(\text{Var}_x(\nabla_\theta l(\theta,x)))^{\frac{1}{2}}=\sqrt{\frac{\alpha}{b}}AB$, i.e., the continuous-time SDE of PGD-AT is,

\begin{equation}
    \begin{gathered}
    d\theta = fdt + \sigma dW_t, \quad \text{where} \quad dW_t=\mathcal{N}(0,Idt) \quad \text{is Wiener process},\\
    f = -(A+(K+\frac{1}{2})\alpha A^2)\theta \qquad \text{and} \quad \sigma = \sqrt{\frac{\alpha}{b}}AB
    \end{gathered}
\end{equation}

We know this is an Ornstein-Uhlenbeck (OU) process and it has a Gaussian distribution as its stationary distribution \cite{Mandt17SGDApproximateBayesian}. Suppose the stationary distribution of this stochastic process has covariance $\Sigma$. Using the same technique in \cite{Mandt17SGDApproximateBayesian,He19ControlBatch}, it is straightforward to verify that $\Sigma$ is explicit and the norm of $\Sigma$ is positively correlated with $\frac{\alpha}{b}$ and norm of $B$ (see e.g. Section 3.2 in \cite{Mandt17SGDApproximateBayesian}). Specifically, its explicit form is $\Sigma=\frac{\alpha}{2b}E$, where $E=A^2H\hat{A}^{-1}$ and $\hat{A}\triangleq A+(K+\frac{1}{2})\alpha A^2$. The proof is that, the above is an OU process and thus its analytic solution is $$\theta_t=\exp(-\hat{A}t)\theta_0+\sqrt{\frac{\alpha}{b}}\int_0^t\exp[-\hat{A}(t-s)]ABdW_s$$. By algebraic transformations, we could verify $\hat{A}\Sigma+\Sigma\hat{A}=\frac{\alpha}{b}A^2H$, where $H=BB^T$. When $\Sigma$ is symmetric (as it is covariance matrix), we get $\Sigma=\frac{\alpha}{2b}E$.

With Lemma \ref{PAC_Bayes_theorem}, we are ready to prove the last statement in Theorem \ref{pgd_at_generalization_theorem}. The posterior covariance is Gaussian with covariance $\Sigma$. We assume a plain Gaussian prior distribution $\mathcal{N}(\theta_0,\lambda_0I_d)$. The density of prior and posterior distributions:
\begin{equation}
    \begin{gathered}
    f_P=\frac{1}{\sqrt{2\pi\det(\lambda_0I_d)}}\exp\Big\{-\frac{1}{2}(\theta-\theta_0)^T(\lambda_0I_d)^{-1}(\theta-\theta_0)\Big\}\\
    f_Q=\frac{1}{\sqrt{2\pi\det(\Sigma)}}\exp\Big\{-\frac{1}{2}\theta^T\Sigma^{-1}\theta\Big\}
    \end{gathered}
\end{equation}

We observe the upper bound $\mathcal{G}$ in Lemma \ref{PAC_Bayes_theorem}, the only term that correlates with $\Sigma$ is $\text{KL}(Q||P)$. Therefore, we calculate their $\text{KL}(Q||P)$ as follows:
\begin{equation}
    \begin{gathered}
    \text{KL}(Q||P)=\int\Big(\frac{1}{2}\log\frac{\lvert\lambda_0I_d\rvert}{\lvert\Sigma\rvert}-\frac{1}{2}\theta^T\Sigma^{-1}\theta\\
    +\frac{1}{2}(\theta-\theta_0)^T(\lambda_0I_d)^{-1}(\theta-\theta_0)\Big)f_Q(\theta)d\theta\\
    =\frac{1}{2}\Big\{\text{tr}\big((\lambda_0I_d)^{-1}\Sigma\big)+\theta_0^T(\lambda_0I_d)^{-1}\theta_0-d+\log\frac{\lvert\lambda_0I_d\rvert}{\lvert\Sigma\rvert}\Big\}\\
    =\frac{1}{2\lambda_0}\theta_0^T\theta_0-\frac{d}{2}+\frac{d}{2}\log\lambda_0+\frac{1}{2\lambda_0}\text{tr}(\Sigma)-\frac{1}{2}\log\lvert\Sigma\rvert
    \end{gathered}
\end{equation}

The second statement indicates increasing diffusion would scale $\Sigma$ up to $c\Sigma$, where $c>1$, the problem is now how does $\text{KL}(Q||P)$ change with $c\Sigma$. The only terms that will change are $\frac{1}{2\lambda_0}\text{tr}(\Sigma)-\frac{1}{2}\log\lvert\Sigma\rvert$. Note that $\frac{1}{2\lambda_0}\text{tr}(\Sigma)$ will become $\frac{c}{2\lambda_0}\text{tr}(\Sigma)$, and $-\frac{1}{2}\log\lvert\Sigma\rvert$ will become $-\frac{1}{2}\log\lvert\Sigma\rvert-\frac{d}{2}\log c$. As $d$ is number of parameters and is potentially extremely large, $\frac{c}{2\lambda_0}\text{tr}(\Sigma)-\frac{1}{2}\log\lvert\Sigma\rvert-\frac{d}{2}\log c$ will be smaller than $\frac{1}{2\lambda_0}\text{tr}(\Sigma)-\frac{1}{2}\log\lvert\Sigma\rvert$. Therefore, $\text{KL}(Q||P)$ will decrease. And consequently $\mathcal{G}$ will decrease. More specifically, we plug in the explicit form of $\Sigma$, and can calculate the derivative $$\frac{\partial \text{KL(Q||P)}}{\partial r}=\frac{r}{4\lambda_0}\text{tr}(E)-\frac{d}{2}\log(\frac{r}{2})-\frac{\log|E|}{2}$$, where $r=\frac{\alpha}{b}$. We could easily verify $\frac{\partial \text{KL(Q||P)}}{\partial r}<0$ when $d>\frac{r\cdot\text{tr}(E)}{2\lambda_0}$. Thus, for sufficiently large $d$, larger $\frac{\alpha}{b}$ results in smaller $\mathcal{G}$. Similarly, we could show for norm of $B$. We complete our proof of the last statement.

\end{document}